\documentclass[12pt,reqno]{amsproc}

\usepackage[margin=1in]{geometry}
\usepackage[utf8]{inputenc}
\usepackage[T1]{fontenc}
\usepackage[all,2cell,cmtip]{xy}
\UseAllTwocells
\usepackage{graphicx,mathrsfs}
\usepackage{longtable}
\usepackage{float}
\usepackage{wrapfig}
\usepackage{rotating}
\usepackage[normalem]{ulem}
\usepackage{amsmath}
\usepackage{amsthm}
\usepackage[british]{babel}
\usepackage{textcomp}
\usepackage{marvosym}
\usepackage{wasysym}
\usepackage{amssymb}
\usepackage[hidelinks,colorlinks=true,linkcolor=blue,citecolor=blue]{hyperref}
\usepackage{color}
\usepackage{enumerate}
\definecolor{bg}{rgb}{0.95,0.95,0.95}
\usepackage{mathtools}
\usepackage{mathpazo}
\usepackage{csquotes}
\usepackage{mathdots}
\usepackage{subcaption}

\usepackage{tikz-cd}
\usepackage{blkarray}

\usepackage{xcolor}

\usepackage{enumitem}

\newtheorem{theorem}{Theorem}[section]
\newtheorem*{theorem*}{Theorem}
\newtheorem*{observation*}{Observation}
\newtheorem{lemma}[theorem]{Lemma}
\newtheorem{corollary}[theorem]{Corollary}
\newtheorem{proposition}[theorem]{Proposition}
\theoremstyle{definition}
\newtheorem{definition}[theorem]{Definition}
\newtheorem{remark}[theorem]{Remark}
\newtheorem{example}[theorem]{Example}

\newcommand{\mbb}[1]{\mathbb{#1}}
\newcommand{\mbf}[1]{\mathbf{#1}}
\newcommand{\mca}[1]{\mathcal{#1}}

\newcommand{\op}{\mathrm{op}}
\newcommand{\id}{\mathrm{id}}

\DeclareMathOperator{\Span}{span}

\DeclareMathOperator{\Tot}{Tot}
\DeclareMathOperator{\Tar}{Tar}
\DeclareMathOperator{\Nullity}{null}

\DeclareMathOperator{\ran}{img}

\newcommand{\R}{\mathbb{R}}
\newcommand{\F}{\mathbb{K}}
\newcommand{\cF}{\mathscr{F}}
\newcommand{\cV}{\mathbf{V}}
\newcommand{\cB}{\mathscr{B}}

\newcommand{\Gr}{\text{\bf Gr}}
\newcommand{\cU}{\mathscr{U}}
\newcommand{\bA}{\mathbf{A}}
\newcommand{\wbA}{\widehat{\mathbf{A}}}
\newcommand{\cA}{\mathscr{A}}
\newcommand{\bS}{\mathbf{S}}
\newcommand{\bT}{\mathbf{T}}
\newcommand{\bC}{\mathbf{C}}
\newcommand{\bF}{\mathbf{F}}

\newcommand{\cI}{\mathscr{I}}

\newcommand{\fdHilb}{\text{\bf FdHilb}}
\newcommand{\Comb}{\text{\bf N}}
\newcommand{\MixedComb}{\text{\bf M}}
\newcommand*\diff{\mathop{}\!\mathrm{d}}

\renewcommand{\phi}{\varphi}

\newcommand{\ip}[1]{\langle{#1}\rangle}
\newcommand{\set}[1]{\left\{{#1}\right\}}
\newcommand\blfootnote[1]{%
  \begingroup
  \renewcommand\thefootnote{}\footnote{#1}%
  \addtocounter{footnote}{-1}%
  \endgroup
}

\title{
Quiver Laplacians and Feature Selection
}
\author{Otto Sumray\(^{1, 2}\), Heather A. Harrington\(^{2, 3, 4, 5}\), Vidit Nanda\(^2\)}

\begin{document}
\begin{abstract}
  The challenge of selecting the most relevant features of a given dataset arises ubiquitously in data analysis and dimensionality reduction. However, features found to be of high importance for the entire dataset may not be relevant to subsets of interest, and vice versa. Given a feature selector and a fixed decomposition of the data into subsets, we describe a method for identifying selected features which are compatible with the decomposition into subsets. We achieve this by re-framing the problem of finding compatible features to one of finding sections of a suitable quiver representation. In order to approximate such sections, we then introduce a Laplacian operator for quiver representations valued in Hilbert spaces. We provide explicit bounds on how the spectrum of a quiver Laplacian changes when the representation and the underlying quiver are modified in certain natural ways. Finally, we apply this machinery to the study of peak-calling algorithms which measure chromatin accessibility in single-cell data. We demonstrate that eigenvectors of the associated quiver Laplacian yield locally and globally compatible features.
\end{abstract}

\maketitle
\blfootnote{Email of corresponding author: \url{otto.sumray@maths.ox.ac.uk}}
\blfootnote{\(^1\)\textsc{Ludwig Cancer Research, University of Oxford}}
\blfootnote{\(^2\)\textsc{Mathematical Institute, University of Oxford}}
\blfootnote{\(^3\)\textsc{Max Planck Institute 
of Molecular Cell Biology and Genetics}}
\blfootnote{\(^4\)\textsc{Centre for Systems Biology Dresden}}
\blfootnote{\(^5\)\textsc{Faculty of Mathematics, Technische Universit\"at Dresden}}

\markleft{\MakeUppercase{Otto Sumray, Heather A. Harrington, Vidit Nanda}}

\vspace{-.4in}

\section{Introduction}

When working with large quantities of unstructured data, it is often convenient -- and on occasion, indispensable -- to employ a family of real-valued {\em features} defined on the dataset. It has, for instance, become standard practice within natural language processing to embed words into Euclidean space based on their relative proximity to each other in a substantial corpus of text \cite{mikolov2013word2vec}. In this case, the coordinates serve as features. Regardless of how such an embedding has been obtained, one immediately seeks to simplify matters by identifying those features which are most relevant to the task at hand. There are several ways of quantifying the relevance of a given feature --- one may use domain-specific context, optimise a statistical quantity such as variance, or exploit the presence of labelled training data, etc. This process of {\em feature selection} forms a cornerstone of data analysis. Examples of well-known feature selectors include principal components~\cite{pearsonPCA}, LASSO~\cite{tibshirani1996LASSO} and recursive feature elimination~\cite{guyon2002RFE}, to name but a few.

We focus here on feature selection in the presence of an auxiliary decomposition of the given dataset. As a toy example of the situation which we have in mind, let us consider the plots in Figure \ref{fig:new-example} of two features defined on datasets which have been decomposed into two overlapping subsets $A$ and $B$.

\begin{figure}[ht]
    \centering
    \includegraphics[width=\textwidth]{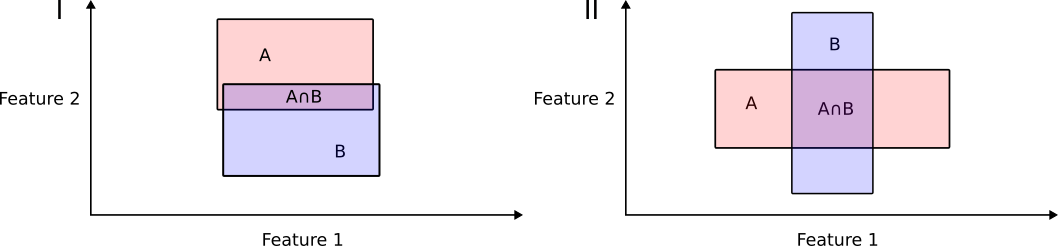}
    \caption{Two features defined on a dataset equipped with a decomposition into two overlapping subsets.}
        \label{fig:new-example}
\end{figure}
\noindent Suppose we rank the relevance of features by their variance. In Figure~\ref{fig:new-example}.I, feature 1 is most relevant for subsets \(A\), \(B\), and \(A \cap B\), but would not be more relevant than feature 2 on their union \(A \cup B\). On the other hand in Figure~\ref{fig:new-example}.II,
feature 1 is most relevant for \(A\), and feature 2 for \(B\), but both are equally relevant for \(A \cap B\) and \(A \cup B\). Such unfortunate inconsistencies arise quite frequently, especially in biological contexts --- we are often interested in features defined on datasets that are naturally organised by types, sub-types and so forth, such as disease or cell types. For such data, any feature selection process (which remains agnostic to the decomposition) is liable to select features that are irrelevant to certain subsets while ignoring features that are relevant to certain subsets.

\subsection*{This Paper} Here we present a principled framework for selecting features on a finite set $X$ that are consistent with respect to a decomposition of $X$ into a {\em cover} $\cU$. Crucially, the subsets of $X$ which comprise $\cU$ are allowed to admit arbitrary overlaps. Given the immense diversity of available feature selectors and their inner workings, we adopt an abstract approach. Let $\F$ denote either the field $\R$ of real numbers or the field $\mathbb{C}$ of complex numbers. For our purposes, a feature selector $\bS$ is any deterministic process that
\begin{enumerate}
    \item accepts as input some finite collection $\cF$ of functions $X \to \F$, and
    \item outputs a subspace $\bS^\cF_X$ of the $\F$-linear span $\cV(\cF)$ of $\cF$.
\end{enumerate} 
We assume an inner product structure on $\cV(\cF)$, but impose no further constraints on $\bS$; as such, the results and constructions of this paper may be applied verbatim to any of the standard feature selectors. Similarly, there are no constraints on the cover $\cU$ besides the requirement that $X = \bigcup_{U \in \cU} U$. 

For the purposes of these introductory remarks, let us call a nonempty subset $\sigma \subset \cU$ {\em admissible} whenever the intersection $|\sigma| := \bigcap_{U \in \sigma} U$ is nonempty. Our quest here is to isolate the largest subspace of $\bS_X^\cF$ that is compatible with respect to the decomposition of $X$ into admissible subsets. Two forms of compatibility are studied here: {\em local} and {\em global}. Given an admissible subset $\sigma \subset \cU$, a feature $v:|\sigma| \to \R$ in $\bS^\cF_{|\sigma|}$ is locally compatible with $\cU$ if for every admissible $\tau \subset \cU$ with $|\tau| \supset |\sigma|$, there is some feature $v^+:|\tau| \to \R$ in $\bS^\cF_{|\tau|}$ whose restriction to $|\sigma|$ coincides with $v$. In other words, a locally compatible feature is one which continues to be selected as we pass to larger admissible subsets. Conversely, consider a collection of features $\set{v_\sigma}$, one for each admissible $\sigma$. Such a collection is globally compatible with $\cU$ if for each pair $|\sigma| \subset |\tau|$ of admissible subsets, the orthogonal projection $\pi^\cF_\tau:\cV(\cF) \twoheadrightarrow \bS^\cF_{|\tau|}$ sends $v_\sigma$ to $v_\tau$. Thus, every globally compatible feature is generated by functions $\set{v_U:U \to \R \mid U \in \cU}$ such that the orthogonality relation \[\left(v_U - v_{U'}\right) \perp \bS^\cF_{U \cap U'}\] holds in $\cV(\cF)$ for every $U,U' \in \cU$.
 
 The first question addressed in our work here is: 
 \begin{quote}
     {\em Fix a feature selector $\bS$ and a cover $\cU$ of $X$. Given an input $\cF$ family of functions $X \to \R$, to what extent can we efficiently discover all of the locally and globally compatible features output by $\bS$?}
 \end{quote}
We approach this problem by recasting it in the language of quiver representations. The underlying quiver $Q^\cU$ is a directed graph whose vertices are admissible sets, with a single edge $\sigma \to \tau$ whenever $|\sigma| \subset |\tau|$. The output of $\bS$ with input $\cF$ induces a {representation} $\bA_\bullet = \bA^{\bS,\cF}_\bullet$ of $Q$ valued in the category of finite dimensional Hilbert spaces. Explicitly, each vertex $\sigma$ is assigned the Hilbert space  $\bA_\sigma := \bS^\cF_{|\sigma|}$, while each edge $\sigma \to \tau$ is assigned the linear map $\bA_{\sigma \to \tau}:\bA_\sigma \to \bA_\tau$ obtained by restricting $\pi^\cF_\tau$ to $\bS^\cF_{|\sigma|}$. We show that the {\bf sections} of this representation (as introduced in  \cite{seigalPrincipalComponentsQuiver2022} and generalised to the multilinear setting in \cite{muller}), recover all the globally compatible features. Similarly, sections of this representation over certain augmented sub-quivers yield all the locally compatible features of $\bS$. 

The presence of noise severely impairs the compatibility framework described above. The main difficulty is that one never obtains an exact equality between (projections of) features selected over an admissible set $|\tau|$ and features selected over an admissible subset $|\sigma|$. The best-case scenario is that features in $\bA_{\sigma}$ lie {\em near}\footnote{Here proximity is determined by the inner product metric on $\cV(\cF)$.} the restriction of some feature in $\bA_{\tau}$. In fact, exact equality may not even be desirable --- often, features are highly correlated, and collapsing together such features will further reduce the dimension of the space of selected features. Thus, there are many compelling reasons to seek a less rigid compatibility framework built around approximate sections of quiver representations. With this consideration in mind, the second question addressed in this paper is:
\begin{quote}
   {\em To what extent can we define and efficiently compute the approximate sections of a quiver representation?}
\end{quote}
We address this challenge by introducing a new Laplacian operator.

The {\bf quiver Laplacian} $L$ introduced here is a (Hermitian, positive semidefinite) endomorphism of the {\em total space} $\Tot(\bA_\bullet) := \prod_\sigma \bA_{\sigma}$. In the special case where every $\bA_{\sigma}$ is one-dimensional and all orthogonal projection maps are identities, $L$ coincides with the ordinary graph Laplacian of $Q^\cU$. We note that Laplacians have been defined for several discrete algebraic/combinatorial objects: from the original case of graphs \cite{chung1997spectral}, to simplicial complexes \cite{goldbergCombinatorialLaplaciansSimplicial2002, hirani2003discrete}, hypergraphs \cite{chung1992laplacian} and principal bundles on graphs \cite{gao2021geometry}. Most closely related to our quiver Laplacian are the analogous operators defined for cellular sheaves in \cite{hansenSpectralTheoryCellular2019} (but see also \cite{ghrist2020cellular, parada2020quiver}). The space of sections of a quiver representation coincides precisely with the kernel of the affiliated quiver Laplacian. Armed with the quiver Laplacian, we are able to define, for every $\epsilon > 0$, the {\em $\epsilon$-approximate sections} of a quiver representation as those normalised vectors $x \in \Tot(\bA^{\bS,\cF}_\bullet)$ which satisfy $\ip{x,Lx} \leq \epsilon$. Thus, the spans of eigenvectors of $L$ associated to small nonzero eigenvalues of $L$ delineate special families of approximate sections. Motivated by these considerations, we establish a host of fundamental properties for quiver Laplacians and their spectra. 

Finally, we use the spectrum of an appropriate quiver Laplacian to study feature selection (as given by a {\em peak-calling} algorithm) for chromatin accessibility\footnote{Chromatin accessibility measures the extent to which regions within the genome (among a specific family of cells) are reachable by cellular machinery, such as transcription factors.} data from \cite{satpathyMassivelyParallelSinglecell2019}. Our dataset comes from single-cell sequencing. This type of data has several properties which make the Laplacian-driven approach particularly appropriate:
\begin{enumerate}
\item The dimensionality is
massive --- a na\"ive estimate
assumes that every genomic position might be accessible, giving around 3 billion features.
Therefore, feature selection (via peak-calling)  becomes a crucial intermediate step towards making the analysis tractable.
\item Different cell types have different relevant features, and it is important to make the feature selection process compatible with the types of cells under consideration. This makes it necessary to consider peaks not only for the entire dataset, but also for certain distinguished subsets.
\item In many relevant cases where the data arise from diseases such as cancer or from the immune system (e.g. T cells), the cells might exhibit plasticity and their type is therefore not well-defined. In such cases, our distinguished subsets might overlap and should be modelled by a cover rather than a partition.
\item There is some stochasticity in the output of peak-calling algorithms --- in particular, the accessible genomic positions are only known approximately. Thus, the accessibility status of nearby positions is not independent; we capture this phenomenon via the inner product structure on feature space.
\end{enumerate}

\subsection*{Summary of Results}

Let $Q$ be a finite quiver with vertex set $Q_0$, edge set $Q_1$ and maps $s,t:Q_1 \to Q_0$ which specify the source and target vertices of each edge. We consider representations of $Q$ --- each such representation $\bA_\bullet$ assigns a finite-dimensional (real or complex) Hilbert space $\bA_v$ to every vertex $v \in Q_0$ and a linear map $\bA_e:\bA_{s(e)} \to \bA_{t(e)}$ to every edge $e \in Q_1$. We write $L_{\bA}$ for the Laplacian of $\bA_\bullet$ (see Definition \ref{def:quivlap} below for details), and arrange its eigenvalues as
\[
0 \leq \lambda_1(\bA) \leq \lambda_2(\bA) \leq \cdots \leq \lambda_n(\bA),
\]
where $n := \dim \Tot(\bA_\bullet)$ is the total dimension of $\bA_\bullet$. 

Here is a simplified version of our first main result, which provides upper bounds on Laplacian eigenvalues of a representation $\bA'_\bullet$ in terms of eigenvalues of $\bA_\bullet$ whenever there exists a family of vertex-indexed linear maps $\bA_v \to \bA'_v$.
\begin{theorem*} [I] Let $\bA_\bullet$ and $\bA'_\bullet$ be two representations of $Q$, and consider any vertex-indexed collection $\tau = \set{\tau_v:\bA_v \to \bA'_v \mid v \in Q_0}$ of linear maps. Write $q$ for the dimension of $\ker \tau$ viewed as a (block diagonal) map $\Tot(\bA_\bullet) \to \Tot(\bA'_\bullet)$. There exist positive constants $\alpha, \beta, \gamma$ -- each dependent on $\tau$ -- such that the inequality
\[
\lambda_k(\bA') \leq \alpha \cdot \lambda_{k+q}(\bA) + \beta \cdot \sqrt{\lambda_{k+q}(\bA)} + \gamma
\]
holds for every $k$ in $\set{1,2,\ldots,n-q}$.
\end{theorem*}
\noindent The full statement of this result, along with explicit formulas for $\alpha, \beta$ and $\gamma$, is recorded in Theorem \ref{thm:trans-quiver-eigen}. If $\tau$ is a morphism of quiver representations $\bA_\bullet \to \bA'_\bullet$, then both $\beta$ and $\gamma$ vanish and we are left with a much simpler bound $\lambda_k(\bA') \leq \alpha \cdot \lambda_{k+q}(\bA)$. 

Our next main result takes the form of a {\em spectral stability} result for feature selectors. The main challenge here is that the total spaces of two representations $\bA^{\bS,\cF}_\bullet$ and $\bA^{\bT,\cF}_\bullet$ may have different dimensions, so even the number of eigenvalues of the corresponding Laplacians could be different. Nevertheless, let $\mu_{\bS}$ and $\mu_{\bT}$ be the probability measures on $\R$ given by taking the average of Dirac measures concentrated at the eigenvalues $\set{\lambda_i(\bA^{\bS,\cF})}$ and $\set{\lambda_j(\bA^{\bT,\cF})}$, respectively. We obtain, in Corollary \ref{cor:specstab}, the following bound on the 1-Wasserstein distance $W_1(\mu_{\bS},\mu_{\bT})$ between these two measures.
\begin{theorem*}[II] Let $\bS$ and $\bT$ be two feature selectors defined on a finite set $X$ which is equipped with a cover $\cU$. Fix a finite collection $\cF$ of functions $X \to \F$. For each admissible set $\sigma$, define 
\[
c_\sigma := \left|\dim \bS^\cF_{|\sigma|} - \dim \bT^\cF_{|\sigma|}\right| \quad \text{ and } \quad \epsilon_\sigma := \text{\rm dist}_{\Gr}\left(\bS^\cF_{|\sigma|}, \bT^\cF_{|\sigma|}\right),  
\]
where the latter expression invokes the  Grassmannian distance between subspaces of $\cV(\cF)$. Define $c := \sum_\sigma c_\sigma$ and $\epsilon := \max_\sigma \epsilon_\sigma$. There exist trigonometric polynomials $f$ and $g$ such that
\[
W_1\left(\mu_{\bS},\mu_{\bT}\right) \leq \frac{|Q^\cU_1|}{m} \Big[f(\epsilon)\cdot c + g(\epsilon) \cdot (m-c)\Big],
\]
where $|Q_1^\cU|$ is the number of edges in $Q^\cU$ and $m = \max\{\dim\Tot(\bA^{\bS,\cF}_\bullet),\dim\Tot(\bA^{\bT,\cF}_\bullet)\}$.
\end{theorem*}
\noindent We note that $c = 0$ if and only if the two feature selectors assign equidimensional spaces to each admissible set. Explicit formulas for $f$ and $g$ are provided in Section \ref{sec:stability}. Crucially, we have $g(0) = 0$; in the equidimensional case, the first term in our bound vanishes and the second term approaches $0$ as $\epsilon \to 0$.

Our next main result provides several flavours of eigenvalue interlacing inequalities for quiver Laplacians which hold when quivers (and the overlaid representations) are modified in certain natural ways.
\begin{theorem*}[III]
Let $\bA_\bullet$ be a representation of $Q$. The following inequalities hold for appropriate integers \(i\). The constants \(k, r, w_1, w_2\) which appear below depend on \(\bA_\bullet\) as well as the modification performed to create a new representation $\bA_\bullet'$ of a new quiver $Q'$.
\begin{enumerate}
    \item If \(Q^\prime\) is obtained by removing edges from \(Q\)
    and \(\bA^\prime_\bullet\) is the restriction of \(\bA_\bullet\)
    to \(Q^\prime\), then
    \[\lambda_{k+i}(\bA_\bullet) 
    \leq \lambda_{k+i}(\bA^\prime_\bullet)
    \leq \lambda_{k+r+i}(\bA_\bullet).\]
    \item If \(Q^\prime\) is obtained by removing vertices (and incident edges) from \(Q\)
    and \(\bA^\prime_\bullet\) is the restriction of \(\bA_\bullet\)
    to \(Q^\prime\) then
        \[\lambda_{i}(\bA_\bullet) - w_1
    \leq \lambda_{i}(\bA^\prime_\bullet)
    \leq \lambda_{k+i}(\bA_\bullet) - w_2.\]
    \item If \(Q^\prime\) and \(\bA^\prime_\bullet\)
    are the result of performing an admissible homotopy
    of \(Q\), then
        \[\varphi^{-2}\lambda_{i}(\bA_\bullet) 
    \leq \lambda_{i}(\bA^\prime_\bullet)
    \leq \varphi^2 \lambda_{i}(\bA_\bullet)\]
    where \(\varphi\) is the golden ratio.
    \item If \(Q^\prime\) and \(\bA^\prime_\bullet\)
    are the result of performing an admissible Kron
    reduction
    of \(\bA_\bullet\), then
        \[\lambda_{i}(\bA_\bullet) 
    \leq \lambda_{i}(\bA^\prime_\bullet)
    \leq \lambda_{i + r}(\bA_\bullet).\]
\end{enumerate}
\end{theorem*}
\noindent Detailed versions of these results, including explicit formulas for various constants as well as precise definitions of admissible homotopy and Kron reduction, are located in Section \ref{sec:eigenvalues}. In order to state and prove these results in an appropriately general setting, we define {\em sheaves} on quivers --- these are almost identical to cellular sheaves \cite{curry2014sheaves, cgn, hansenSpectralTheoryCellular2019} on graphs, except that our graphs are allowed to have self-loops.

As mentioned above, much of our work here has been motivated by the desire to better understand and improve peak-calling algorithms on datasets $X$ which come naturally equipped with a decomposition $\cU$ into subsets. For this purpose, it is convenient to combine the quiver representations whose sections give locally and globally compatible features of a given feature selector $\bS$ into a single representation of a larger quiver. This {\em combined representation}, denoted $\Comb^{\bS,\cF}_\bullet$, is described in Section \ref{sec:global-sections}. By design, its sections correspond to features which are both locally and globally compatible with $\cU$ along $\cF$, and we naturally seek its approximate sections. Unfortunately, extracting spectral data for its Laplacian presents a serious computational challenge --- both the underlying quiver and the total dimension are typically much larger than $Q^\cU$ and $\bA^{\bS,\cF}_\bullet$. The third main result of this paper is a remedy which takes the form of a {\em reduction theorem}.

\begin{theorem*}[IV] Let $\bS$ be a feature selector on a set $X$ equipped with a cover $\cU$ such that \(Q^\cU\) is weakly connected. Let $\Comb^{\bS,\cF}_\bullet$ be the combined representation associated to a finite collection $\cF$ of functions $X \to \F$. Fix a vertex $\sigma \in Q^\cU_0$. There exists a new quiver $Q'$ with a single vertex and $|Q^\cU_0|$ edges along with a representation $\bA'_\bullet$ of $Q'$ satisfying two properties:
\begin{enumerate}
\item For each vertex $\tau$ of $Q^\cU$, let $\id_{\tau}$ be the identity map on $\bS^\cF_{|\tau|}$ and let $\iota^\cF_\tau:\bS^\cF_{|\tau|} \hookrightarrow \cV(\cF)$ denote the inclusion map (which is adjoint to $\pi^\cF_\tau)$. Then, the Laplacian of $\bA'_\bullet$ is 
\[
L_{\bA'} = \sum_{\tau \in Q^\cU_0} \left[\id_\sigma - \pi^\cF_\sigma\iota^\cF_\tau\pi^\cF_\tau\iota^\cF_\sigma\right]
\]
\item Let $n$ be the total dimension of $\Comb^{\bS,\cF}_\bullet$. There exist positive constants $C_1$ and $C_2$, which depend only on $\bA^{\bS,\cF}_\bullet$, such that the inequalities
\[
C_1 \cdot \lambda_i(\Comb^{\bS,\cF}_\bullet) \leq \lambda_i(\bA'_\bullet) \leq C_2 \cdot \lambda_{i+k}(\Comb^{\bS,\cF}_\bullet)
\]
hold for every $i$ in $\set{1,2,\ldots,n}$, with $k := n - \dim \bS^{\cF}_{|\sigma|}$.
\end{enumerate}
\end{theorem*}
\noindent The full statement of this result is recorded in Theorem \ref{thm:computing-lim-e}.

Finally, we examined a single-cell chromatin accessibility (ATAC-seq) dataset from \cite{satpathyMassivelyParallelSinglecell2019}; here $X$ is a dataset
consisting of around \(3 \times 10^4\) tumour-infiltrating T cells, while the input $\cF$ consists of 20,000 binary-valued features $X \to \R$. To build a cover $\cU$ of $X$, we used a standard notion of genomic proximity between cells, constructed a $15$-nearest neighbour graph, and expanded the communities of that graph by one hop. This produces a quiver $Q^\cU$ with 87 vertices. We then ran the MACS2 peak-calling algorithm \cite{zhang2008model} to find the 1,000 most accessible regions for each admissible set and built a quiver Laplacian for a variant of the combined representation of $\bS$ (as described in Theorem~\ref{thm:computing-dual-lim-e}). 

\begin{observation*}[V]
The first 3,000 eigenvectors of the quiver Laplacian described above had the following properties:
\begin{enumerate}
\item The majority of eigenvectors displayed a tight grouping of genomic positions (around 150 consecutive base-pairs out of approximately 3 billion) across the vertices.
\item Most eigenvectors were supported in most of the vector spaces assigned to vertices of the quiver. In particular, one such eigenvector that was supported on every vertex,
lay near the gene {\it FAM72D}, which is involved in the cell cycle and hence active for every type of T cell.
\item Conversely, the support of a handful of eigenvectors was localised to a very small subset of vertices. One such eigenvector lay near the gene {\it STAT3}, which regulates Th1, Th17 and Treg cells; this eigenvector was supported on just 3 vertices, each of which corresponds to admissible subpopulations particularly rich in these three types of T cells.
\end{enumerate}
\end{observation*}
\noindent Thus, eigenvectors of the quiver Laplacian were able to isolate relevant and consistent meta-features across different scales within the given dataset.

\subsection*{Outline} The paper is structured as follows:
\S\ref{sec:feature-selection} describes how a feature selector together with a cover of the data yields a quiver representation. In \S\ref{sec:compatsec} we describe both local and global compatibility of features with respect to the cover. We also describe how these features can be computed simultaneously as sections of a larger quiver representation.
\S\ref{sec:laplacian} introduces the quiver Laplacian 
and how its eigenvectors may be used to approximate the space of sections.
\S\ref{sec:stability} studies how a transformation of quiver representations relates
the spectra of the respective representations and bounds their spectral distance.
This then gives a stability of the spectrum of the representation associated to a feature selector.
\S\ref{sec:eigenvalues} bounds the changes in the spectrum of a quiver Laplacian when the underlying quiver
undergoes various operations.
\S\ref{sec:reduction} then uses these operations to simplify the 
computation of approximate sections of a feature selector.
Finally, in \S\ref{sec:sc-atac} we use quiver Laplacians to extract (approximately) compatible peaks in single-cell chromatin accessibility data.

\section*{Acknowledgements}
OS is supported by Ludwig Cancer Research. VN and HAH are grateful for the support provided by the UK Centre for Topological Data Analysis EPSRC grant EP/R018472/1. HAH gratefully acknowledges funding from the Royal Society. VN is partially supported by US AFOSR grant FA9550-22-1-0462. OS would like to thank Renee Hoekzema, Ka Man (Ambrose) Yim, Phil Xie, and Xin Lu for many helpful conversations. For the purpose of Open Access, the authors have applied a CC BY public copyright licence to any Author Accepted Manuscript (AAM) version arising from this submission.

\section{From Feature Selection to Quiver Representations}
\label{sec:feature-selection}
Given a finite set $X$ and a field $\F$, let $X^*$ be the $\F$-vector space which consists of all functions $X \to \F$. Denote by $\text{Sub}(X^*)$ the collection of all subsets of $X^*$ (not necessarily subspaces), and for each such subset $\cF$ write $\cV(\cF)$ to indicate the subspace of $X^*$ given by the $\F$-linear span of $\cF$. Let $\Gr_k(X^*)$ be the  Grassmannian of $k$-dimensional subspaces of $X^*$ for each $k \geq 0$, and consider the disjoint union
\[
\Gr_\infty(X^*) := \coprod_{k \geq 0} \Gr_k(X^*).
\] 
Here is the main object of study in this work.
\begin{definition}\label{def:fsp}
	A {\bf feature selector} on $X$ is any map $\bS:\text{Sub}(X^*) \to \Gr_\infty(X^*)$, such that for each input $\cF \subset X^*$, the corresponding output $\bS_X^\cF$ is a subspace of $\cV(\cF)$. 
\end{definition}
\noindent 
We now fix a feature selector $\bS$ on $X$; readers may wish to keep the following concrete example in mind throughout the remainder of \S\ref{sec:feature-selection}.  
\begin{example} Among the most well-known and ubiquitous feature selectors is {\em principal components analysis} \cite{pcabook} --- here,
	\begin{itemize}
		\item $X$ is a finite origin-centered subset of some real Euclidean space $\R^n$,
		\item $\F$ is the field $\R$ of real numbers, and
		\item $\cF$ is the set $\set{x \mapsto \ip{x,p}}$ of maps given by taking inner products with the unit vectors $p$.
	\end{itemize} 
For a fixed $k \ll n$, the output $\bS^\cF_X$ is the vector space spanned by the $k$ inner product functions (or equivalently, the $k$ unit vectors) along which the variance in $X$ is maximised.
\end{example}

 Our goal here is to find a principled framework for relating the vector subspaces $\bS^\cF_Y$ to each other for a fixed $\cF$ as $Y$ ranges over subsets of $X$. Already for principal components analysis, it is clear that there need not be any coherent relationship between $\bS^\cF_X$ and $\bS^\cF_Y$ for arbitrary $Y \subset X$ --- for instance, neither one is guaranteed to be a subspace of the other within $\cV(\cF)$. It therefore becomes necessary to constrain the family of subsets under consideration. Here we will describe how to relate the subspaces $\bS^\cF_Y$ for $Y$ ranging over the lattice of subsets generated from a chosen {\em cover} $\cU$ of $X$. We recall that any such $\cU$ is a finite collection of nonempty subsets satisfying $X = \bigcup_{U \in \cU} U$. 

\begin{definition} \label{def:coverquiver}
	The {\bf quiver associated to a cover $\cU$} of $X$ is the directed graph $Q^\cU$ whose
	\begin{enumerate}
		\item vertices are subsets $\sigma \subset \cU$ whose {\em support} $|\sigma| := \bigcap_{U \in \sigma} U$ is nonempty; and,
		\item there is a unique directed edge $\sigma \to \tau$ whenever $\sigma$ properly contains $\tau$ (and hence \(|\sigma| \subset |\tau|\)).
	\end{enumerate}
\end{definition}

We now re-assemble the algebraic data of $\bS$ to produce a {\em representation} of $Q^\cU$, i.e., an assignment of vector spaces to vertices and linear maps to edges \cite{schiffler}. Implicit in the construction below is the choice of an inner product structure on $X^*$, which allows us to define adjoints of linear maps, and hence, orthogonal projections onto subspaces. For each vertex $\sigma$ of $Q^\cU$, 
let $\iota^\cF_\sigma:\bS^\cF_{|\sigma|} \hookrightarrow \cV(\cF)$
denote the inclusion map and
let $\pi^\cF_\sigma:\cV(\cF) \twoheadrightarrow \bS^\cF_{|\sigma|}$ denote its adjoint, i.e., the orthogonal projection in $X^*$ with respect to our chosen inner product.

\begin{definition} \label{def:srep}
For each subset $\cF\subset X^*$, the {\bf $\bS$-representation} of $Q^\cU$ along $\cF$ consists of the following assignments $\bA_\bullet = \bA^{\bS,\cF}_\bullet$:
\begin{enumerate}
	\item every vertex $\sigma$ of $Q^\cU$ is assigned the vector space $\bA_\sigma := \bS^{\cF}_{|\sigma|}$, which is a subspace of $\cV(\cF|_{|\sigma|})$ and hence\footnote{Here we adopt the convention that any function $f:|\sigma| \to \F$ extends to all of $X$ by setting $f = 0$ on $X \setminus |\sigma|$. Conversely, any function $X \to \F$ automatically restricts to a function on $|\sigma|$.} of $\cV(\cF)$; moreover,
	\item every edge $\sigma \to \tau$ of $Q^\cU$ is assigned the $\F$-linear map $\bA_{\sigma \to \tau}:\bA_\sigma \to \bA_\tau$
          given by \(\pi^\cF_\tau \circ \iota^\cF_\sigma\).
\end{enumerate}
\end{definition}

It may be worth noting that the vertices of $Q^\cU$ are simplices in the {\em nerve} of the cover $\cU$ \cite{eilenberg2015foundations}, and every edge $\sigma \to \tau$ in $Q^\cU$ corresponds to a face relation (where the simplex $\tau$ lies in the boundary of the simplex $\sigma$). In particular, it follows from the definition of $Q^\cU$ above that the existence of adjacent edges $\sigma \to \tau \to \gamma$ forces the existence of the edge $\sigma \to \gamma$. We  call $\bA_\bullet$ a {\bf sheaf} on the nerve of $\cU$ if it satisfies the following {\em associativity criterion}: for every such pair of adjacent edges, the map $\bA_{\sigma \to \gamma}$ must equal the composite $\bA_{\tau \to \gamma} \circ \bA_{\sigma \to \tau} $. In general, we do not expect this associative property to hold when $\bA_\bullet$ arises from a feature selector. The next section contains our attempt to rectify this shortcoming.

\section{Compatibility and Sections}\label{sec:compatsec}

Here we use the $\bS$-representation to define and study two notions of compatibility for a feature selector $\bS$ on $X$ against a cover \(\cU\) when invoked with input $\cF \subset X^*$:
\begin{enumerate}
	\item Fix a vertex $\sigma \in Q^\cU_0$. An element \(v_\sigma \in \bA^{\bS, \cF}_{\sigma}\) is {\em locally compatible} with $\cU$ if
\[
  \iota^{\cF}_\sigma(v_\sigma) = \iota^{\cF}_{\tau} \circ \bA_{\sigma \to \tau} (v_\sigma)
\]
holds for every edge \(\sigma \to \tau\) in $Q_1$. \\
\item A collection \(\set{v_\sigma \in \bA_{\sigma} \mid \sigma \in Q^{\cU}_0}\)
is {\em globally compatible} with $\cU$ if the equality \[\bA_{\sigma \to \tau} (v_\sigma) = v_\tau\] holds for every edge \(\sigma \to \tau\) in $Q^\cU$.
\end{enumerate}

This notion of local compatibility checks whether or not (the orthogonal projections of) a given feature $v_\sigma$ which has been selected by $\bS$ over $|\sigma|$ continue to be selected over all greater subsets $|\tau| \supset |\sigma|$ generated by the cover $\cU$. In this sense, local compatibility tests the robustness of selected features as we decrease the depth of the cover.
Already for principal component analysis, it is clear that not every feature will be locally compatible, i.e., in general  \(\bA_{\sigma \to \tau}\) has a non-trivial kernel. 
Global compatibility, on the other hand, tests features horizontally across the quiver --- to be globally compatible, the vectors of a family $\set{v_\sigma}$ must map coherently onto each other under the linear maps of $\bA_\bullet$ as we  traverse zigzag paths in $Q^\cU$ of the form
\[
\sigma_1 \to \tau_1 \gets \sigma_2 \to \tau_2 \gets \cdots \gets \sigma_k \to \tau_k.
\] 

\begin{figure}[h!]
    \centering
    \includegraphics[width=.85\textwidth]{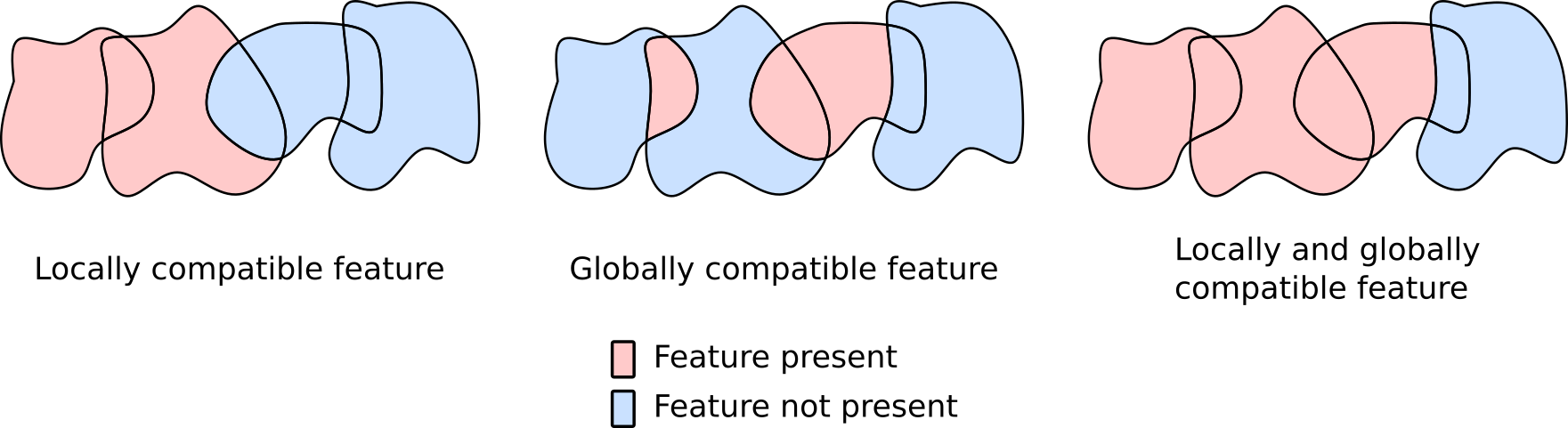}
    \caption{Example of different compatibility conditions.}
\end{figure}

We defer the study of local compatibility for now, and will focus on global compatibility.

\subsection{Global Compatibility and Sections}
\label{sec:global-sections}

Let $Q = (s,t:Q_1 \to Q_0)$ be a finite quiver and let \(\bA_\bullet\) be a representation of $Q$ valued in the category $\fdHilb(\F)$ of finite-dimensional Hilbert spaces over the field $\F \in \set{\R,\mathbb{C}}$. The {\em total space} of $\bA_\bullet$ is defined as the product
\begin{align}\label{eq:totspc}
\Tot(\bA_\bullet) := \prod_{i \in Q_0} \bA_i.
\end{align}
Since $Q$ is finite, $\Tot(\bA_\bullet)$ inherits a finite-dimensional Hilbert space structure from its $\bA_i$ factors. Motivated by the global compatibility criterion from Definition \ref{def:compat}, we highlight a relevant subspace of the total space below.

\begin{definition}
  \label{def:section}
  A tuple $\gamma := (\gamma_i \in \bA_i \mid i \in Q_0)$ in $\Tot(\bA_\bullet)$ is called a {\bf section} of $\bA_\bullet$ if for every edge $e$ in $Q_1$ we have 
  \(
    \bA_{e}(\gamma_{s(e)}) = \gamma_{t(e)}.
   \)
\end{definition}
\noindent Sections of quiver representations were introduced in \cite{seigalPrincipalComponentsQuiver2022} along with an effective algorithm for their computation --- see \cite[Sec 5.2]{seigalPrincipalComponentsQuiver2022}. Sections are closely related to the compatibility criteria from Definition \ref{def:compat}. In particular, if $\bS$ is a feature selector on a set $X$, then the sections of \(\bA^{\bS, \cF}_\bullet\) are exactly globally compatible features of
\(\bS\) against the cover \(\cU\). 
Sections can also be used to compute locally compatible features,
as we will now describe.

\subsection{Local Compatibility}

\begin{definition} \label{def:compat}
  A feature selector $\bS$ on $X$ is {\bf locally compatible} with $\cU$ on $\cF \subset X^*$ if the following triangle of vector spaces commutes for each edge $\sigma \to \tau$ of $Q^\cU$:
\[
\begin{tikzcd}
\bA^{\bS,\cF}_\sigma \arrow[rr, "\bA^{\bS,\cF}_{\sigma \to \tau}"] \arrow[dr, "\iota^\cF_\sigma", swap]
& &
\bA^{\bS,\cF}_\tau \arrow[dl, "\iota^\cF_\tau"] \\
& \cV(\cF) &
\end{tikzcd}
\]
Equivalently, the identity $\iota_\sigma^\cF = \iota_\tau^\cF \circ \pi^\cF_\tau \circ \iota^\cF_\sigma$ holds for every edge $\sigma \to \tau$ of $Q^\cU$.
\end{definition} 

It will be useful to rephrase this compatibility condition entirely within the realm of quiver representations by blowing up the above triangle to a square. To this end, we introduce another representation $\bC_\bullet = \bC^{\cF}_\bullet$ of $Q^\cU$ which may be  associated to every $\cF\subset X^*$ (and which does not depend on $\bS$). This {\em constant representation} assigns $\bC_\sigma := \cV(\cF)$ to every vertex $\sigma$ and the identity map to each edge $\sigma \to \tau$. For each vertex $\sigma$, there is an evident inclusion map
\[
  \iota^\cF_\sigma: \bA^{\bS,\cF}_\sigma \hookrightarrow \bC^\cF_\sigma,
\]
since the codomain is $\cV(\cF)$ and the domain is its subspace $\bS^\cF_{|\sigma|}$. Now, for every edge $\sigma \to \tau$ we have a diagram of four linear maps:

\begin{equation}\label{eq:comp-qrep}
\begin{tikzcd}
  \bA^{\bS,\cF}_\sigma
  \arrow[d,  "\iota^\cF_\sigma", swap]
  \arrow[rr, "\bA^{\bS,\cF}_{\sigma \to \tau}"]
  &&
  \bA^{\bS,\cF}_\tau   
  \arrow[d, "\iota^\cF_\tau"]
  \\
  \bC^\cF_\tau
  \arrow[rr, "\bC^\cF_{\sigma \to \tau}" , swap]
  &&
  \bC^\cF_\sigma
\end{tikzcd}
\end{equation}

The vertical maps $\iota^\cF$ prescribe a {\em morphism of quiver representations} from  $\bA^{\bS,\cF}_\bullet$ to  $\bC^\cF_\bullet$ if and only if the above diagram commutes for every edge $\sigma \to \tau$ in $Q^\cU$.
By definition of $\bC^\cF_\bullet$, the bottom edge of this diagram is precisely $\cV(\cF) \stackrel{=}{\longrightarrow} \cV(\cF)$, so this square is an elementary modification of the triangle from Definition \ref{def:compat}. Thus, we have arrived at the following straightforward result.

\begin{proposition}\label{prop:comp=quivmor}
The feature selector $\bS$ is locally compatible with $\cU$ on $\cF \subset X^*$ if and only if the collection of maps 
\[
\set{\iota^\cF_\sigma: \bA^{\bS,\cF}_\sigma \hookrightarrow \bC^\cF_\sigma}
\] indexed by vertices $\sigma$ of $Q^\cU$ prescribe a morphism of quiver representations $\bA^{\bS,\cF}_\bullet \to \bC^\cF_\bullet$.
\end{proposition}

As described earlier, a morphism \(\tau\) between
quiver representations \(\bA_\bullet\) and \(\bA^\prime_\bullet\)
is a collection
\[
\tau = \set{\tau_v: \bA_v \to \bA^\prime_v \mid v \in Q_0}
\]
of linear maps \(\tau_v\) 
such that
\begin{equation}
\label{eqn:morphism}
    \bA^\prime_e \circ \tau_{s(e)} = \tau_{t(e)} \circ \bA_{t(e)}
\end{equation}
for each edge \(e \in Q_1\).
The commutivity condition in (\ref{eqn:morphism}) can be too strong
for quiver representations obtained from data,
so we will define a {\em transformation} \(\tau\) between
quiver representations as a collection of linear maps 
\[
\tau = \set{\tau_v: \bA_v \to \bA^\prime_v \mid v \in Q_0}
\]
without the commutivity condition.

\begin{definition}
  Let \(Q\) be a quiver.
  For each vertex \(u \in Q_0\)
  define the {\bf floret} of based at \(u\) 
  to be the following quiver \(F^{Q, u}\):
  \begin{enumerate}
      \item the set of vertices is the disjoint union
  \[
  F^{Q, u}_0 = \set{u} \sqcup \set{v_e \mid (e : u \to v) \in Q_1}.
  \]
  In other words, there is a distinguished copy of $u$, which we label \(u_0\) henceforth, and an additional object denoted $v_e$ for every edge in \(Q_1\) from \(u\) to \(v\).
  
\item Define the set of edges of \(F^{Q, u}\) as 
\[
F^{Q, u}_1 = \set{e_{LD}, e_{DL} : u_0 \to v_e \mid (e : u \to v) \in Q_1}.
\]
In other words, for each edge \(e: u \to v\) in \(Q\) 
there are two distinct edges \(e_{LD}\) and \(e_{DL}\)
from \(u_0\) to \(v_e\) in \(F^{Q, u}\).
  \end{enumerate}
  Now suppose \(\bA_\bullet\) and \(\bA^\prime_\bullet\)
  are representations of the quiver \(Q\)
  with a transformation
  \(\tau: \bA_\bullet \to \bA^\prime_\bullet\).
  Define the following representation \(\bF^{\tau, u}_\bullet\)
  of \(F^{Q, u}\):
  \begin{align*}
  \bF^{\tau, u}_{u_0} &= \bA_u &
  \bF^{\tau, u}_{v_e} &= \bA^\prime_v \\
  \bF^{\tau, u}_{e_{LD}} &= \bA^\prime_e \circ \tau_u &
  \bF^{\tau, u}_{e_{DL}} &= \tau_{t(e)} \circ \bA_e.
  \end{align*}
  See Figure~\ref{fig:floret} for an illustration.

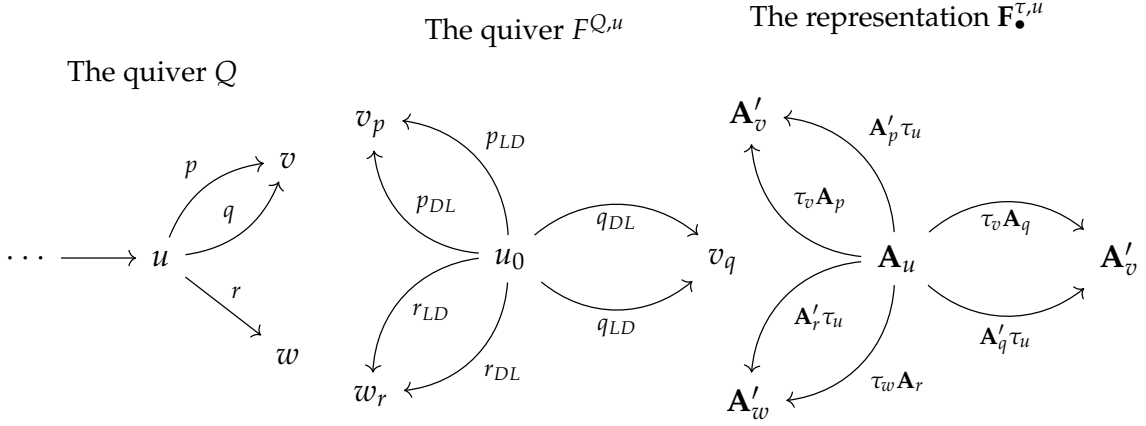
\begin{figure}[ht]
\begin{subfigure}{0.3\textwidth}
\caption*{The quiver \(Q\)}
\[
\begin{tikzcd}[ampersand replacement=\&]
  \&\& v \\
{\cdots}\&u\&   \\
  \&\& w 
\arrow[from=2-1, to=2-2]
\arrow["p", from=2-2, to=1-3, bend left=30]
\arrow["q", from=2-2, to=1-3, bend right=30]
\arrow["r", from=2-2, to=3-3]
\end{tikzcd}
\]
\end{subfigure}%
\begin{subfigure}{0.3\textwidth}
\caption*{The quiver \(F^{Q, u}\)}
  \[\begin{tikzcd}[ampersand replacement=\&, row sep=scriptsize]
	v_p \& \\
	\\
	\& {u_0} \&\& v_q \\
	\\
	w_r
	\arrow["{p_{DL}}"', bend left=40, from=3-2, to=1-1]
	\arrow["{p_{LD}}"', bend right=40, from=3-2, to=1-1]
	\arrow["{q_{DL}}"', bend left=40, from=3-2, to=3-4]
	\arrow["{q_{LD}}"', bend right=40, from=3-2, to=3-4]
	\arrow["{r_{DL}}", bend left=40, from=3-2, to=5-1]
	\arrow["{r_{LD}}", bend right=40, from=3-2, to=5-1]
\end{tikzcd}\]
\end{subfigure}%
\begin{subfigure}{0.3\textwidth}
  \caption*{The representation \(\bF^{\tau, u}_\bullet\)}
    \[\begin{tikzcd}[ampersand replacement=\&, row sep=scriptsize]
	\bA^\prime_{v} \& \\
	\\
	\& \bA_u \&\& \bA^\prime_v \\
	\\
	\bA^\prime_w
	\arrow["{\tau_v \bA_p}"', bend left=40, from=3-2, to=1-1]
	\arrow["{\bA^\prime_p \tau_u}"', bend right=40, from=3-2, to=1-1]
	\arrow["{\tau_v \bA_q}"', bend left=40, from=3-2, to=3-4]
	\arrow["{\bA^\prime_q \tau_u}"', bend right=40, from=3-2, to=3-4]
	\arrow["{\tau_w \bA_r}", bend left=40, from=3-2, to=5-1]
	\arrow["{\bA^\prime_r \tau_u}", bend right=40, from=3-2, to=5-1]
\end{tikzcd}\]
\end{subfigure}
\caption{An illustration of the floret based at \(u\) and its representation.}
\end{figure}
\label{fig:floret}
\end{definition}

Constructing the floret now allows us to compute locally compatible features: in the case where the transformation is \(\iota^\cF: \bA^{\bS, \cF}_\bullet \to \bC^\cF_\bullet\) as in Proposition \ref{prop:comp=quivmor},
the sections \(\Gamma(\bF^{\iota^{\cF}\sigma}_\bullet)\)
are exactly the locally compatible features at \(\sigma\).

\subsection{Bicompatible features}

In Section \ref{sec:laplacian} below, we will consider the problem of discovering features which satisfy approximate versions of the local and global compatibility criteria. For this purpose, it will be helpful to construct a single quiver whose sections correspond to features which are both locally and globally compatible. The first step in this construction is a method for gluing quivers and their representations. Let \(Q = (s,t:Q_1 \to Q_0)\) and \(Q' = (s',t':Q'_1 \to Q'_0)\) be two quivers. 

\begin{definition}
   For each subset $R \subset Q_0 \times Q_0'$, let \(Q \cup_{R} Q'\) be the {\bf merged quiver} whose 
  \begin{enumerate}
  	\item vertex set is the quotient $V := (Q_0 \sqcup Q'_0)/R$,
  	\item edge set is the disjoint union $E := Q_1 \sqcup Q'_1$,
   \end{enumerate}
and source/target maps are given as follows. Let $\rho$ be the composite $Q_0 \hookrightarrow Q_0 \sqcup Q_0' \twoheadrightarrow V$, and define $\rho':Q_0' \to V$ similarly. If $e \in E$ lies in $Q_1$, then its source and target vertices are $\rho \circ s(e)$ and $\rho \circ t(e)$; and if $e$ lies in $Q_1'$, then its source and target vertices are $\rho' \circ s'(e)$ and $\rho' \circ t'(e)$ respectively.  
\end{definition}

\noindent Fix representations \(\bA_\bullet\) of a quiver $Q$ and \(\bA'_\bullet\) of another quiver $Q'$.  

\begin{definition}
Let $R \subset Q_0 \times Q'_0$ be any subset for which we have $\bA_i = \bA'_j$ whenever $(i,j)$ lies in $R$. The {\bf merged representation} of $\bA_\bullet$ and $\bA'_\bullet$, denoted $(\bA \cup_R \bA')_\bullet$, is the following representation of $Q \cup_R Q'$. To vertices, it assigns
\[
(\bA \cup_R \bA')_i := \begin{cases} \bA_i & i \in Q_0 \text{ and } (i,j) \notin R \text{ for any }j \in Q'_0 \\
\bA'_i & i \in Q'_0 \text{ and } (j,i) \notin R \text{ for any }j \in Q_0 \\
\bA_i = \bA'_j & \text{if } (i,j) \in R \\
\bA_j = \bA'_i & \text{if } (j,i) \in R
\end{cases}
\]
And on edges, we have 
\[
(\bA \cup_R \bA')_e := \begin{cases} \bA_e & \text{if }e \in Q_1, \\
\bA'_e & \text{if } e \in Q_1'.
\end{cases}
\]
\end{definition}

One may embed both $\Tot(\bA_\bullet)$ and $\Tot(\bA'_\bullet)$ within the total space of the merged representation in the natural way, and it is readily seen that the space of sections of $(\bA \cup_R \bA')_\bullet$ is (isomorphic to) the intersection $\Gamma(\bA_\bullet) \cap \Gamma(\bA'_\bullet)$.

\begin{figure}[h!]
\includegraphics[scale=.85]{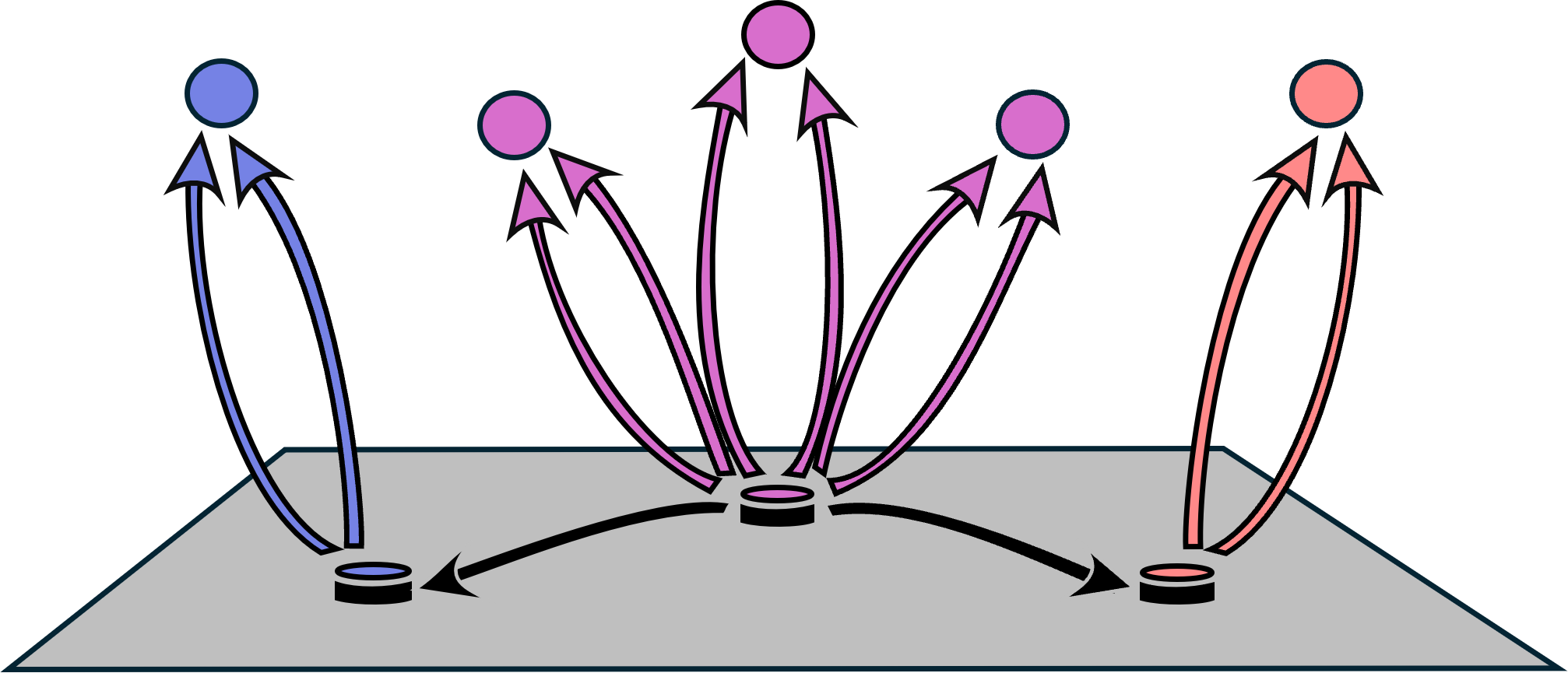}
\caption{Illustration of the merged quiver}
\end{figure}

Returning to the case of interest, let $\bS$ be a feature selector on a finite set $X$; consider a cover $\cU$ of $X$ and a subset $\cF \subset X^*$. Let $\widehat{Q}^\cU$ be the quiver given by the disjoint union of florets $\cI^\sigma$ indexed by vertices $\sigma \in Q^\cU_0$, and let $\wbA_\bullet^{\bS,\cF}$ be the representation of $\widehat{Q}^\cU$ given by the floret functors $\cB^\sigma$. We define $R \subset (Q^\cU)_0 \times \widehat{Q}^\cU_0$ as the collection 
\[
R := \set{(\sigma,\sigma_0) \mid \sigma \in Q^\cU_0},
\]
where $\sigma_0$ is the central vertex of $\cI^\sigma$.

\begin{definition} \label{def:combinedrep}
The {\bf combined $\bS$-representation} (of $Q^\cU \cup_R \widehat{Q}^\cU$, along $\cF$) is the merged representation $\Comb^{\bS,\cF}_\bullet := \left(\bA^{\bS,\cF} \cup_R \wbA^{\bS,\cF}\right)_\bullet$.
\end{definition}
\noindent By design, the sections of $\Comb^{\bS,\cF}_\bullet$ correspond to relevant features which are both locally and globally compatible with $\cU$. We call these sections the {\bf bicompatible features} of $\bS$.

\begin{example}\label{eg:bicompat}
    Suppose \(\cU = \set{U_1, U_2}\) is a cover of \(X\),
such that \(U_1 \cap U_2 \neq \varnothing\).
Define \(\tau_1 = \set{U_1}, \tau_2 = \set{U_2}\)
and \(\sigma = \set{U_1, U_2}\).
The quiver \(Q^{\cU}\) is
\[
\begin{tikzcd}
    \tau_1 & \sigma \arrow[r] \arrow[l] & \tau_2 .
\end{tikzcd}
\]
Consider a bicompatible feature, i.e., a section \(x \in \Gamma(\Comb^{\bS, \cF}_\bullet)\). By definition, we have \(\iota^\cF_{\tau_2}(x_{\tau_2}) = \iota^\cF_\sigma(x_\sigma) = \iota^\cF_{\tau_1}(x_{\tau_1})\),
hence 
\[
\Gamma(\Comb^{\bS, \cF}_\bullet) 
\cong \iota^\cF(\bS^{\cF}_{|\tau_1|}) \cap \iota^\cF(\bS^{\cF}_{|\sigma|}) \cap \iota^\cF(\bS^{\cF}_{|\tau_2|}).
\]
Thus, bicompatible features are the ones which are relevant for every \(U \in \cU\).
\end{example}

\subsection{Dual Representations} 
\label{sec:dual-reps}

If one is also interested in features which are relevant for some -- but not all -- \(U \in \cU\), then it becomes necessary to consider a different notion of bicompatibility. Given a quiver \(Q = (s,t:Q_1 \to Q_0)\), its {\em dual} \(Q^*\) is the quiver with the same vertices and edges but with source and target maps interchanged, i.e., \(s^* := t\) and \(t^* := s\). Every representation \(\bA_\bullet\) of \(Q\) valued in \(\fdHilb(\F)\) induces a unique {\em dual representation} \(\bA_\bullet^*\) of \(Q^*\), which is obtained by taking the adjoint of every edge map: namely, \((\bA^*)_e := (\bA_e)^*\) for each \(e \in Q_1\). 

When $\bA_\bullet$ is the $\bS$-representation (from Definition \ref{def:srep}) of a feature selector $\bS_\bullet$ along some input $\cF$, the fact that $\iota_\sigma^\cF$ and $\pi_\sigma^\cF$ form an adjoint pair for each vertex $\sigma \in Q^\cU_0$ 
implies that for every edge \(\sigma \to \tau\) in $Q^\cU$, we have 
\[
(\bA^{\bS, \cF})^*_{\tau \to \sigma} = (\pi_\tau^\cF \circ \iota_\sigma^\cF)^* = \pi_\sigma^\cF \circ \iota_\tau^\cF.
\]

\begin{definition} \label{def:dual-combinedrep}
The {\bf mixed $\bS$-representation} (of $(Q^\cU)^* \cup_R \widehat{Q}^\cU$, along $\cF$) is the merged representation $\MixedComb^{\bS,\cF}_\bullet := \left((\bA^{\bS,\cF})^* \cup_R \wbA^{\bS,\cF}\right)_\bullet$.
\end{definition}

Let us revisit Example \ref{eg:bicompat} with a view towards understanding the sections of $\MixedComb^{\bS,\cF}_\bullet$. These sections are given by the direct sum

\[
\Gamma(\MixedComb^{\bS,\cF}_\bullet) \cong 
(\iota_{\tau_1}(\bS^{\cF}_{|\tau_1|}) \cap \iota_{\sigma}(\bS^{\cF}_{|\sigma|})^\perp)
\oplus
(\iota_{\tau_2}(\bS^{\cF}_{|\tau_2|}) \cap \iota_{\sigma}(\bS^{\cF}_{|\sigma|})^\perp)
\oplus \Gamma(\Comb^{\bS, \cF}_\bullet).
\]
In this case,
\(\iota_{\tau_1}(\bS^{\cF}_{|\tau_1|}) \cap \iota_{\sigma}(\bS^{\cF}_{|\sigma|})^\perp\) 
corresponds to features that are relevant to \(U_1\) but not to \(U_1 \cap U_2\).
Note that 
\[
(\iota_{\tau_1}(\bS^{\cF}_{|\tau_1|}) \cap \iota_{\sigma}(\bS^{\cF}_{|\sigma|})^\perp)
\cap
(\iota_{\tau_2}(\bS^{\cF}_{|\tau_2|}) \cap \iota_{\sigma}(\bS^{\cF}_{|\sigma|})^\perp)
\]
might not be trivial, 
i.e. there could be a feature relevant to both \(U_1\) and \(U_2\) but not to \(U_1 \cap U_2\),
and hence this feature corresponds to two independent sections.

\section{The Quiver Laplacian and Approximate Sections}
\label{sec:laplacian}
So far, we have reframed compatibility-testing of feature selectors in terms of finding the sections of a quiver representation. However, the space of sections might be trivial \cite[Sec 5.1]{seigalPrincipalComponentsQuiver2022}. Rather than seeking exact solutions, we endeavour to find approximate sections.

\subsection{The Quiver Laplacian}
Let $\fdHilb(\F)$ denote the category of finite-dimensional Hilbert spaces over the field $\F \in \set{\R,\mathbb{C}}$; every morphism $A:U \to V$ in this category admits an {\em adjoint} morphism $A^*:V \to U$ characterised by $\ip{Au,v} = \ip{u,A^*v}$ for all $(u,v)$ in $U \times V$. Here we fix a finite quiver $Q = \left(s,t:Q_1 \to Q_0\right)$ and consider a representation $\bA_\bullet$ of $Q$ valued in $\fdHilb(\F)$. We recall the {total space} from \eqref{eq:totspc} and define {\em target space} of $\bA_\bullet$ as the direct product
\[
\Tar(\bA_\bullet) := \prod_{e \in Q_1} \bA_{t(e)}.
\]
Since we have assumed that $Q$ is finite, both these spaces inherit a Hilbert space structure from the individual factors of the form $\bA_v$. We will denote the identity map on $\bA_v$ by $\id_v$ rather than $\id_{\bA_v}$.

\begin{definition}
  Given an \(\fdHilb(\F)\)-valued quiver representation \(\bA_\bullet\) of \(Q\), its
  {\bf boundary operator}
  \[
    B_\bA :\Tot(\bA_\bullet) \xrightarrow{} \Tar(\bA_\bullet)
  \]
  is the linear map given in component form by
\[
  (B_{\bA})_{e, v} =
  \begin{cases}
    \bA_{e} - \id_{{v}} & \text{if } s(e) = t(e) = v, \\
    \bA_{e}                  & \text{if } s(e) = v \text{ and } t(e) \neq v,\\
    -\id_{v}         & \text{if } s(e) \neq v \text{ and } t(e) = v,\\
    \mbf{0} & \text{otherwise}.
  \end{cases}
\]
\end{definition}

If the quiver \(Q\) is finite, then it follows from a simple calculation that \(\ker{B_\bA} \) is isomorphic (as a Hilbert space) to the space \(\Gamma(Q;\bA_\bullet)\) of $\bA_\bullet$'s sections. From the boundary operator of a quiver representation, we can construct another associated operator, the Laplacian, which will allow us to compute
approximate sections.

\begin{definition}\label{def:quivlap}
 The {\bf Laplacian} of $\bA_\bullet$ is the endomorphism
  \(
    L_\bA: \Tot(\bA_\bullet) \to \Tot(\bA_\bullet)
  \)
  given by
  \[L_{\bA} = B_{\bA}^* B_{\bA}.\]
  In component form, we have
  \[
    (L_\bA)_{v, v} = 
       \sum_{\substack{e \in Q_1 \\ s(e) = v}}{\bA_e^* \bA_e} +
      \sum_{\substack{e \in Q_1 \\ t(e) = v}}{\id_{v}} -
      \sum_{\substack{e \in Q_1 \\ s(e) = t(e) = v}}{[\bA_e^* + \bA_e]}
   \]
   and for \(u \neq v\)
   \[
     (L_\bA)_{u, v} =
     - \sum_{\substack{e \in Q_1 \\ s(e) = u \\ t(e) = v}}{\bA_e^*}
     - \sum_{\substack{e \in Q_1 \\ s(e) = v \\ t(e) = u}}{\bA_e}.     
   \]
\end{definition}

Since \(\ker{L_\bA} = \ker{B_\bA}\), the kernel of \(L_\bA\) also computes the sections of \(\bA_\bullet\).
The main advantage of considering $L_\bA$ rather than $B_\bA$ in this context is that it is a Hermitian and positive semi-definite matrix. As such, it enjoys favourable spectral properties; in particular, all of
its eigenvalues are all real and non-negative.
We shall order these eigenvalues in increasing fashion, with multiplicity:
\[
\lambda_1(L_\bA) \leq \dots \leq \lambda_n(L_\bA)
\]
where \(n := \dim \Tot(\bA_\bullet)\). If \(\phi: \bA^1_\bullet \to \bA^2_\bullet\)
is an isomorphism of quiver representations
 with each \(\phi_v\) a unitary map then
 \[
      L_{\bA^1} = \Phi \circ L_{\bA^2} \circ \Phi^*,
 \]
 where \(\Phi = \prod_{v \in Q_0}{\phi_v}\),
 thus we may speak unambiguously of the eigenvalues \(\lambda_i(\bA_\bullet) := \lambda_i(L_\bA)\)
 of a quiver representation \(\bA_\bullet\).

\subsection{Approximate Sections}
To measure the distance in $\Tot(\bA_\bullet)$ between an arbitrary vector $x$ and the space $\Gamma(\bA_\bullet)$ of sections, we consider the {\bf Dirichlet energy}:
\[
  \mca{E}_{\bA}(x) := \langle x, L_{\bA}x \rangle =
  \sum_{e \in Q_1}{\|\bA_e(x_{s(e)}) - x_{t(e)}\|^2_{\bA_{t(e)}}} \in \mbb{R}_{\geq 0}
\]
Here \(\langle -, - \rangle\) is the induced inner product on \(\Tot(\bA_\bullet)\), while
and \(x_v\) denotes the component of \(x\) in \(\bA_{v}\), and $\| - \|$ denotes the norm induced by the inner product. As \(L_{\bA}\) is Hermitian, we can form an orthonormal eigenbasis
\(e_1, \dots, e_n\) where \(e_i\) corresponds to \(\lambda_i(\bA_\bullet)\). Writing \(x\) in this basis, we have
\[
  \mca{E}_\bA(x) = \mca{E}_\bA\left(\sum_{i}{x_i e_i}\right) = \sum_i{\lambda_i(\bA_\bullet)x_i^2}.
\]
Thus, \(\sqrt{\mca{E}_\bA(x)}\) is the $\lambda_\bullet$-weighted distance from \(x\) to \(\Gamma(\bA_\bullet)\);  and moreover, \[\sum_{i: \lambda_i(\bA) = 0}{x_i e_i}\] is the closest vector to \(x\) in \(\Gamma(\bA_\bullet)\).
This motivates the following definition. (A similar notion for cellular sheaves was introduced in~\cite{joslyn2020sheaf}).

\begin{definition}
\label{dfn:approx-section}    
For a quiver \(Q\) and an \(\fdHilb(\F)\)-valued
representation \(\bA_\bullet\) of \(Q\),
and {\em \(\varepsilon\)-approximate section}
of \(\bA_\bullet\) is an \(x \in \Tot(\bA_\bullet)\)
such that \(\|x\|=1\) and 
\(
\mca{E}_{\bA}(x) \leq \varepsilon.
\)
\end{definition}
\noindent 
By Theorem~\ref{thm:rayleigh},
if \(x \in \Span\{e_1, \dots, e_k\}\)
and \(\|x\| = 1\),
then \(x\) is a \(\lambda_k\)-approximate section.
Note this is not exhaustive: there
can exist \(\lambda_k\)-approximate sections
of \(\bA_\bullet\) that are not in 
the span of the first \(k\) eigenvectors,
e.g. \(\lambda_1 x_1 e_1 + \lambda_{k + 1} x_{k+1} e_{k+1}\) for small enough \(x_{k+1}\).
Regardless, to compute approximate sections of a quiver representation
one can compute eigenvectors of \(L_\bA\)
with small eigenvalues.
As we have connected approximate sections
with the spectra of the quiver representation,
we now want to understand how the spectra changes
if we perturb the quiver representation (\S\ref{sec:stability}), 
or the underlying quiver (\S\ref{sec:eigenvalues}).

\section{Spectral Stability of Feature Selectors}
\label{sec:stability}
  Let \(Q\) be a quiver and suppose \(\bA_\bullet\) and \(\wbA_\bullet\)
  are two representations of \(Q\).
  Suppose \(\tau\) is a collection \(\{(\tau_v: \bA_v \to \wbA_v) : v \in Q_0\}\) of linear maps.
  Define the {\em defect} of \(\tau\) at an edge \(e \in Q_1\) to be
  \[
    \partial(\tau, e) = \left\|\wbA_e \tau_{s(e)} - \tau_{t(v)} \bA_e\right\|,
  \]
   Here, and henceforth, the expression $\|-\|$ when applied to a linear map between Hilbert spaces will always indicate the operator norm --- for $M:V \to W$, the norm $\|M\|$ equals the supremum of $\|Mx\|_W$ over unit vectors $x \in V$. Now the {\em total defect} of \(\tau\) is:
  \[
    \partial(\tau) = \left(\sum_{e \in Q_1}\partial(\tau, e)^2\right)^{\frac{1}{2}}.
  \]
  The defect of \(\tau\) is zero if and only if when
  \(\tau\) prescribes a morphism of quiver representations. Given bases of \(\Tot(\bA_\bullet)\) and \(\Tot(\wbA_\bullet)\),
  we view \(\tau\) as a block diagonal matrix.
  Let \(\tau^+\) denote its Moore-Penrose pseudoinverse. The values \(\|\tau\|\) and \(\|\tau^+\|\) are unique up to unitary transforms of \(\Tot(\bA_\bullet)\) and \(\Tot(\wbA_\bullet)\). Finally, we let
  \[
    \kappa(\tau) := \|\tau\| \phantom{.} \|\tau^+\|.
  \]
  be the {\em generalised condition number} of \(\tau\).

  We seek to describe how the existence of $\tau$ forces a relationship between the Laplacian spectra of $\bA_\bullet$ and $\wbA_\bullet$. The first step in this direction is the following lemma.

  \begin{lemma}
    \label{lem:lower-bound-singular}
    For a matrix \(\tau\), if \(x \in (\ker{\tau})^\perp\) then
    \[
      \|\tau x\| \geq \|\tau^+\|^{-1} \|x\|.
    \]
  \end{lemma}
  \begin{proof}
    Let \(u_1, \dots, u_m\) and \(v_1, \dots, v_n\) be
    left and right singular vectors of \(\tau\) respectively.
    Let \(x = \sum_{i=1}^{n}{x_i v_i}\).
    Then \(\tau x = \sum_{i=1}^{\min\{n , m\}}{\sigma_i x_i u_i}\),
    but as \(x \in \ker{\tau}^\perp\) we have that
    \(x_j = 0\) when \(\sigma_j = 0\) or \(j > \min\{m, n\}\).
    Thus
    \[
      \|\tau x\| \geq
      \left\|\sum_{i=1}^{\min\{n , m\}}{\sigma_+ x_i u_i} \right\| = \sigma_+ \|x\| = \|\tau^+\|^{-1} \|x\|
    \]
    where \(\sigma_+\) is the smallest non-zero singular value of \(\tau\).
  \end{proof}

\noindent The next theorem describes how the existence of a transformation between quiver representations constrains their eigenvalues. In the statement below (and henceforth), we use $\Nullity(f)$ as a shorthand for the dimension of the kernel of a linear map $f:U \to V$ of finite-dimensional Hilbert spaces. 
  
  \begin{theorem}
    \label{thm:trans-quiver-eigen}
  Suppose \(\bA_\bullet\) and \(\wbA_\bullet\)
  are two representations of a quiver \(Q\) whose total spaces have dimensions \(n\) and \(m\) respectively. Let \(\tau = \{\tau_v:\bA_v \to \wbA_v \mid v \in Q_0\}\) be a collection of linear maps, viewed as a single map $\Tot(\bA_\bullet) \to \Tot(\wbA_\bullet)$. Set $q := \Nullity(\tau)$; for every integer \(k\) such that \(1 \leq k \leq n - q\),
  we have 
  \[
    \lambda_k(\wbA) 
   \leq 
   \kappa(\tau)^2 \lambda_{k + q}(\bA) + \partial(\tau) \|\tau^+\| \left[2\kappa(\tau)\lambda_{k + q}(\bA)^{\frac{1}{2}} + \partial(\tau)\|\tau^+\|\right].
  \]
    In particular, if \(\tau\) is a morphism of quiver representations,
    then $\partial(\tau)$ vanishes and hence
   \[
    \lambda_k(\wbA) 
    \leq \kappa(\tau)^2 \lambda_{k + q}(\bA).
  \]
\end{theorem}
\begin{proof}
  Let \(x_1, \dots, x_n\) and \(y_1, \dots, y_m\) be eigenvectors of \(\bA_\bullet\)
  and \(\wbA_\bullet\) respectively, ordered by increasing eigenvalue.
  Suppose \(k\) is an integer such that \(1 \leq k \leq n - q\).
  Let \(l = k + q\) and 
  define \(X_l = \Span\{x_1, \dots, x_l\}\) and \(Y^k = \Span\{y_k, \dots, y_m\}\).
  As \(l = k + q\)
  we have that \(\dim \tau (X_l) \geq k\)
  thus \(\dim \tau(X_l) + \dim Y^k\) is at least \(m + 1\)
  hence \(\tau(X_l)\) and \(Y^k\) intersect non-trivially.
  Thus there exists some non-zero \(y \in \tau(X_l) \cap Y^k\),
  and as \(y \in Y^k\),
  by Theorem~\ref{thm:rayleigh}
  it holds that
  \[
    \lambda_k(\wbA) 
     \leq \frac{1}{\left\|y\right\|^2} \sum_ {e \in Q_1} \left\|\wbA_e y_{s(e)} - y_{t(e)} \right\|^2.
   \]
   Now we shall bound the right-hand side of the above.
   Since \(y \in \tau(X_l)\)
   there exists some \(x \in X_l \cap (\ker{\tau})^\perp\) such that \(\tau x = y\).
   Using Lemma~\ref{lem:lower-bound-singular} have that
   \[
\frac{1}{\left\|\tau x\right\|^2} \sum_{e \in Q_1} \left\|\wbA_e \tau_{s(e)} x_{s(e)} - \tau_{t(e)} x_{t(e)} \right\|^2
\leq \frac{\|\tau^+\|^2}{\left\|x\right\|^2} \sum_{e \in Q_1} \left\|\wbA_e \tau_{s(e)} x_{s(e)} - \tau_{t(e)} x_{t(e)} \right\|^2.
   \]
   Then applying the triangle inquality gives us
   \begin{align*}
& \frac{1}{\left\|\tau x\right\|^2} \sum_{e \in Q_1} \left\|\wbA_e \tau_{s(e)} x_{s(e)} - \tau_{t(e)} x_{t(e)} \right\|^2 \\
       & \qquad \leq \frac{\|\tau^+\|^2}{\left\|x\right\|^2} \sum_{e \in Q_1} \left\|\tau_{t(e)} \bA_e x_{s(e)} - \tau_{t(e)} x_{t(e)} \right\|^2 \\
       & \qquad \qquad + 2 \frac{\|\tau^+\|^2}{\left\|x\right\|^2}\sum_{e \in Q_1}\left\|\tau_{t(e)} \bA_e x_{s(e)} - \tau_{t(e)} x_{t(e)} \right\|
         \left\|\wbA_e \tau_{s(e)}x_{s(e)} - \tau_{t(e)} \bA_ex_{s(e)}\right\| \\
            & \qquad \qquad + \frac{\|\tau^+\|^2}{\left\|x\right\|^2}\sum_{e \in Q_1}\left\|\wbA_e \tau_{s(e)}x_{s(e)} - \tau_{t(e)} \bA_ex_{s(e)}\right\|^2
   \end{align*}
   and then with the Cauchy-Schwarz inequality we have
     \begin{align*}
      & \frac{1}{\left\|\tau x\right\|^2} \sum_{e \in Q_1} \left\|\wbA_e \tau_{s(e)} x_{s(e)} - \tau_{t(e)} x_{t(e)} \right\|^2 \\
    & \qquad \leq \kappa(\tau)^2\frac{1}{\|x\|^2}\sum_{e \in Q_1}\left\|\bA_e x_{s(e)} - \tau_{t(e)} x_{t(e)} \right\|^2 \\
       & \qquad \qquad + \kappa(\tau)\|\tau^+\|\bigg[\frac{1}{\|x\|^2}\sum_{e \in Q_1}\left\|\bA_e x_{s(e)} - \tau_{t(e)} x_{t(e)} \right\|^2\bigg]^{\frac{1}{2}}\partial(\tau) + \|\tau^+\|^2 \partial(\tau)^2
     \end{align*}
     and finally since \(x \in X_l\), we may apply Theorem~\ref{thm:rayleigh} to obtain
     \begin{align*}
       & \frac{1}{\left\|\tau x\right\|^2} \sum_{e \in Q_1} \left\|\wbA_e \tau_{s(e)} x_{s(e)} - \tau_{t(e)} x_{t(e)} \right\|^2 \\
               & \qquad \leq \kappa(\tau)^2 \lambda_l(\bA) + 2\kappa(\tau)\|\tau^+\|\lambda_l(\bA)^{\frac{1}{2}}\partial(\tau) + \|\tau^+\|^2 \partial(\tau)^2
     \end{align*}
     completing the proof.
\end{proof}

We shall now describe how Theorem~\ref{thm:trans-quiver-eigen} gives us a bound on the
spectral pseudo-metric between quiver representations. If the total spaces of \(\bA\) and \(\wbA_\bullet\)
have the same dimension, then we can directly compare the associated Laplacian spectra. Following \cite{gu2015spectral}, we use the Wasserstein metric to perform such comparisons in order to account for the case where total spaces have different dimensions.
\begin{definition} For each real number $r \geq 0$, let \(B(Q, r)\) 
be the set of representations \(\bA_\bullet\) of \(Q\) 
which satisfy \(\max_{e \in Q_1} \|A_{e}\| \leq r\). 
\end{definition}
\noindent By design, if \(\bA_\bullet \in B(Q, r)\), then its eigenvalues are bounded by \(|Q_1|(r + 1)^2\).

 Given a representation \(\bA_\bullet\) of \(Q\)
with total dimension \(n\),
its {\em spectral measure} is
\[
  \mu_{\bA} = \sum_{i = 1}^{n}{\delta_{\lambda_i(\bA)}}
\]
where \(\delta_{\lambda_i(\bA)}\) is the Dirac measure concentrated at \(\lambda_i(\bA)\). Given two probability measures \(\mu_1, \mu_2\) on \(\mbb{R}_{\geq 0}\),
a {\em coupling} \(\gamma\) of \(\mu_1\) and \(\mu_2\)
is any probability measure on \(\mbb{R}_{\geq 0} \times \mbb{R}_{\geq 0}\)
whose marginals on each factor are \(\mu_1\) and \(\mu_2\) respectively.
Denote by \(\Gamma(\mu_1, \mu_2)\) the set of all such couplings.
Then the Wasserstein \(p\)-metric between \(\mu_1\) and \(\mu_2\)
is defined as
\[
  W_p(\mu_1, \mu_2) =
    \inf_{\gamma \in \Gamma(\mu_1, \mu_2)}
    \left(
    \int_{\mbb{R}_{\geq 0} \times \mbb{R}_{\geq 0}}
    \|x - y\|^p \diff\gamma(x, y)
    \right)^{\frac{1}{p}}
\]

Adapting the proof of Theorem 6.3 in~\cite{gu2015spectral} we have the following:
\begin{theorem}
 \label{thm:spectral-distance}
  Suppose \(\bA_\bullet, \wbA_\bullet \in B(Q, r)\)
  with total dimensions \(n\) and \(m\) respectively,
  where \(n \geq m\).
  Suppose \(\tau\) is a non-zero transformation from \(\bA_\bullet\) to \(\wbA_\bullet\).
  If \(\text{\rm rank}(\eta) \geq \Nullity(\eta^*)\) then
  \[
    W_1(\mu_{\bA}, \mu_{\wbA}) \leq C
  \]
  where \(C\) depends on \(\tau, n\) and \(r\).
  Moreover, if \(\|\tau\|, \|\tau^+\| \leq 1 + \varepsilon\) and \(\partial(\tau), \partial(\tau^*) \leq \varepsilon\)
  for some \(\varepsilon > 0\), then we have
  \[
    \lim_{\varepsilon \to 0} C = \frac{2\Nullity(\tau) + 2\Nullity(\tau^*)}{n}|Q_1|(r + 1)^2,
  \]
  which, if \(\Nullity(\tau) + \Nullity(\tau^*) < n/2\), will be less than na\"ive bound of \(|Q_1|(r+1)^2\).
\end{theorem}
\begin{proof}
  Define \(c_{\tau} = \Nullity(\tau) + \Nullity(\tau^*)\).
  We will now define a coupling \(\gamma\) between \(\mu_\bA\) and \(\mu_{\wbA}\).
  For each positive integer $p$, let $[p]$ denote the set $\set{1,2,\ldots,p-1,p}$. Setting \(\ell = \Nullity(\tau^*)\) and  and \(u = n - \Nullity(\tau)\),
  define the following partition of \([n] \times [m]\):
  \begin{align*}
  P_1 &= \set{(i, i) : i \in [u] \setminus [\ell]} \\
  P_2 &= \set{(i, j) : i \in [\ell] \text{ and } j \in [\ell]} \\
    & \qquad \cup \, \set{(i, j) : i \in [\ell] \text{ and } j \in [m] \setminus [u]} \\
    & \qquad \cup \, \set{(i, j) : i \in [n] \setminus [u] \text{ and } j \in [\ell]} \\
    & \qquad \cup \, \set{(i, j) : i \in [n]\setminus [u] \text{ and } j \in [m] \setminus [u]} \\
  P_3 &= \set{(i, j): i \in [\ell] \text{ and } j \in  [u]\setminus [\ell]} \\
  & \qquad \cup \, \set{(i, j): i \in [n] \setminus [u] \text{ and } j \in [u]\setminus [\ell]} \\
  P_4 &= \left([n] \times [m]\right) \setminus \left(P_1 \cup P_2 \cup P_3\right).
  \end{align*}
Now define the desired coupling \(\gamma\) as 
  \[
    \gamma = \sum_{\substack{1 \leq i \leq n \\ 1 \leq j \leq m}} w_{ij} \delta_{(\lambda_i(\bA),\, \lambda_j(\wbA))}
    \qquad \text{where} \qquad
  w_{ij} =
  \begin{cases}
      \frac{1}{n} & (i, j) \in P_1,\\
      \frac{1}{m (n - c_\tau)} & (i, j) \in P_2,\\
      \left(\frac{1}{m} - \frac{1}{n}\right)\frac{1}{n - c_\tau} & (i, j) \in P_3, \\
      0 & (i, j) \in P_4.
  \end{cases}
  \]
Set \(R = |Q_1|(r + 1)^2\), which is an upper bound for the eigenvalues of \(\bA_\bullet\).
  Then using Theorem~\ref{thm:trans-quiver-eigen} for \(\tau\) and \(\tau^*\) and substituting \(R\) gives us
  \[
    \kappa(\tau)^{-2}\lambda_{i - \Nullity(\tau^*)}(\bA_\bullet) - a_{\tau^*, R} \leq \lambda_i(\wbA_\bullet) \leq \kappa(\tau)^2 \lambda_{i + \Nullity(\tau)}(\bA_\bullet) + b_{\tau, R}
  \]
  with the constants
  \begin{align*}
    a_{\tau^*, R} &= 2\kappa(\tau)^{-1}\|\tau^+\|\partial(\tau^*)\sqrt{R} + \kappa(\tau)^{-2}\|\tau^+\|^2\partial(\tau^*)^2, \text{ and} \\
    b_{\tau, R} &= 2\kappa(\tau)\|\tau^+\|\partial(\tau)\sqrt{R} + \|\tau^+\|^2\partial(\tau)^2.
  \end{align*}
  Computing cost of \(\gamma\) gives us
  \begin{align*}
  W_1(\mu_\bA, \mu_{\wbA}) \leq &\frac{1}{n} \Biggl[ \sum_{k = \Nullity(\tau^*) + 1}^{n - \Nullity(\tau)} |\lambda_k(\bA_\bullet) - \lambda_k(\wbA_\bullet)| + c_{\tau}R \Biggr]\\
                         & \leq \frac{1}{n} \Biggl[ \sum_{k = \Nullity(\tau^*) + 1}^{n - \Nullity(\tau)} (\kappa(\tau)^2\lambda_{k + \Nullity(\tau)}(\bA_\bullet) - \kappa(\tau)^{-2}\lambda_{k - \Nullity(\tau^*)}(\bA)) \\
                        & \qquad \qquad \qquad + (n - c_{\tau})(a_{\tau^*, R} + b_{\tau, R}) + c_{\tau}R \Biggr].
  \end{align*}
  Rearranging the sum on the right side yields
  \begin{align*}
    \frac{1}{n} \left[c_{\tau} \kappa(\tau)^2 R + (n - 2 c_{\tau}) (\kappa(\tau)^2 - \kappa(\tau)^{-2})R + (n - c_{\tau})(a_{\tau^*, R} + b_{\tau, R}) + c_{\tau}R \right],
  \end{align*}
which is the desired $C$. 
\end{proof}

Suppose \(A, B\) are a pair of orthogonal matrices with dimensions \(n \times k\) and \(n \times l\) respectively. Then the matrix $A^*B$ is an orthogonal projection \(\ran{B} \to \ran{A}\), where $\ran{A}$ denotes the image of $A$. Let the singular values of \(A^*B\) be \(\sigma_1, \dots, \sigma_r\) where \(r=\min\{k,l\}\). Then the principal angles between \(\ran{A}\) and \(\ran{B}\) are defined by \(\cos{\theta_i} = \sigma_i\) and the Grassmannian metric between \(\ran{A}\) and \(\ran{B}\)
as defined by~\cite{ye2016schubert} is
\[
d_{\Gr_\infty}(\ran{A}, \ran{B})
:= \Bigl(\sum_{i=1}^{r} \theta_i^2\Bigr)^{\frac{1}{2}}.
\]

\begin{corollary}\label{cor:specstab}
  Suppose \(\bS, \bT : \text{\rm Sub}(X^*) \to \Gr_\infty(X^*)\)
  are two feature selectors on \(X\) and suppose
  \(\cU\) is a cover of \(X\) and \(\cF \subset X^*\).
  Suppose that
  \[
  \max_{\sigma \in Q^\cU_0}{d_{\Gr_\infty}(\bS^\cF_{|\sigma|}, \bT^\cF_{|\sigma|})} \leq \varepsilon < \frac{\pi}{2}
  \]
  for some \(\varepsilon > 0\), and define 
  \[c = \sum_{\sigma \in Q_0^\cU}
  |\dim{\bS^\cF_{|\sigma|}} - \dim{\bT^\cF_{|\sigma|}}|.
  \]
  Then, $W_1(\mu_{\bA^{\bS, \cF}}, \mu_{\bA^{\bT, \cF}})$ is no larger than 
  \begin{multline*}
  \frac{4|Q_1|}{n}
  \biggl[
  (1 + \cos^2{\varepsilon}) c 
  + \left(
    2\sec{\varepsilon}\tan{\varepsilon}
    + (1 + \sec^2{\varepsilon})\tan^2{\varepsilon} + \sec^2{\varepsilon} - \cos^2{\varepsilon}
  \right)(n-c)
  \biggr]
\end{multline*}
  where \(n = \max\{\dim{\Tot(\bA^{\bS,\cF}_\bullet)}, \dim{\Tot(\bA^{\bT,\cF}_\bullet})\}\).
\end{corollary}
\begin{proof}
Assume that \(\dim{\bS^\cF_\bullet} \geq \dim{\bT^\cF_\bullet}\).
Define the transformation \(\eta\)
with components \(\eta_\sigma = \pi^{\bT}_\sigma \iota^{\bS}_\sigma\).
Then \(\cos{\varepsilon} \leq \|\eta\| \leq 1\)
and \(1 \leq \|\eta^+\| \leq (\cos{\varepsilon})^{-1}\).
In order to bound \(\partial(\eta)\)
note that
\[
\|\iota^{\bS}_\sigma \pi^{\bS}_\sigma - \iota^{\bT}_\sigma \pi^{\bT}_\sigma\|_2 = \sin{\theta_{\sigma, \max}} \leq \sin{\varepsilon}, 
\]
where \(\theta_{\sigma, \max}\) is the largest principal
angle between \(\bS_{|\sigma|}\) and \(\bT_{|\sigma|}\).
This is known as the projection metric~\cite{ye2016schubert}
where the equality can be seen from Theorem 2.6.1 in~\cite{golubMatrixComputations1996}.
This then implies that, for each edge \(\sigma \to \tau\),
we have
\(
\partial(\eta, \sigma \to \tau) \leq 2\sin{\varepsilon}
\)
thus \(\partial(\eta) \leq 2\sqrt{|Q_1|}\sin{\varepsilon}\).
Similarly, \(\partial(\eta^*) \leq 2\sqrt{|Q_1|}\sin{\varepsilon}\).
Applying Theorem~\ref{thm:spectral-distance}
together with these bounds
provide the result.
\end{proof}

\section{Eigenvalue Inequalities for Quiver Operations}
\label{sec:eigenvalues}
So far, we have described our problem as computing sections of a quiver representation. Here we seek to simplify the representation and the underlying quiver whilst preserving its space of sections and bounding the change in its eigenvalues. We will establish similar eigenvalue inequalities when the quiver itself is deformed via certain natural operations. The most convenient framework for extracting such inequalities is to upgrade the given quiver representation to a sheaf, as described below.

\subsection{Sheaves over Quivers}

Fix a quiver \(Q = (s,t:Q_1 \to Q_0)\). 

\begin{definition}
  An \(\fdHilb(\F)\)-valued {\bf sheaf} \(\cA_\bullet\)
  over \(Q\) consists of:
  \begin{itemize}
  \item
    A finite-dimensional Hilbert space \(\cA_v\) for each vertex \(v \in Q_0\),
  \item
    a finite-dimensional Hilbert space \(\cA_e\) for each edge \(e \in Q_1\), and
  \item
    for each for each edge \(e \in Q_1\),
    a pair of linear maps
    \[\cA_{s(e) \to e}:\cA_{s(e)} \to \cA_{e} \quad \text{ and } \quad \cA_{e \gets t(e)}:\cA_{t(e)} \to \cA_e.\]
  \end{itemize}
\end{definition}

Such sheaves are closely related to representations of quivers. Every representation \(\bA_\bullet\) over \(Q\) induces a sheaf \(\cA_\bullet\) over \(Q\)
by setting \(\cA_{s(e) \to e} = \bA_e\) and \(\cA_{e \leftarrow t(e)} = \id_{{t(e)}}\). Conversely, every sheaf evidently induces a representation on the subdivision of $Q$ where every edge $e:u \to v$ is decomposed into the zigzag $u \to e \gets v$. Thus, there is a natural notion of Laplacian for sheaves over a quiver, a special case\footnote{The only difference between our quiver sheaf Laplacian and the sheaf Laplacians of \cite{hansenSpectralTheoryCellular2019} over one-dimensional cell complexes is that we allow edges to have the same source and target vertex. On the other hand, cellular sheaves and their Laplacians are also defined over regular cell complexes of dimension $> 1$.} of which appears  in~\cite{hansenSpectralTheoryCellular2019}.

\begin{definition}
  Suppose \(\cA_\bullet\) is a \(\fdHilb(\F)\)-valued sheaf over a quiver \(Q\). 
  Define \(\Tot(\cA_\bullet) = \prod_{v \in Q_0} \cA_v\)
  and \(\Tar(\cA_\bullet) = \prod_{e \in Q_1} \cA_e\).
  \begin{enumerate}
   \item The {\bf boundary operator} of $\bA_\bullet$ is the linear map \(B_\cA: \Tot(\cA_\bullet) \to \Tar(\cA_\bullet)\) of \(\cA_\bullet\) given in component form by
  \[
    (B_\cA)_{e, v} =
    \begin{cases}
      \cA_{v \to e} - \cA_{e \leftarrow v} & \text{if } s(e) = t(e) = v, \\
      \cA_{v \to e} & \text{if } s(e) = v \text{ and } t(e) \neq v, \\
      - \cA_{e \leftarrow v} & \text{if } s(e) \neq v \text{ and } t(e) = v, \\
      \mbf{0} & \mathrm{otherwise.}
    \end{cases}
  \]
  \item The {\bf Laplacian} of \(\cA_\bullet\) is then defined as
  \[
    L_{\cA} = B_\cA^* B_\cA.
  \]
  \item Finally, the {\bf Dirichlet energy} for \(\cA\) is the function $\mca{E}_\cA:\Tot(\cA_\bullet) \to \mathbb{R}$ given by
  \[
    \mca{E}_{\cA}(x) = \sum_{e \in Q_1}\left\|\cA_{s(e) \to e}(x_{s(e)}) - \cA_{e \leftarrow t(e)}(x_{t(e)})\right\|^2_{\cA_{e}}.
  \]
  \end{enumerate}
\end{definition}
If one constructs a sheaf $\cA_\bullet$ from a quiver representation \(\bA_\bullet\) of \(Q\) as described above, then the sheaf boundary operator \(B_\cA\) coincides exactly with the quiver boundary operator \(B_\bA\); therefore the Laplacians of $\bA$ and $\cA$, as well as their spectra, also coincide in this case.

\subsection{Operations on the Quiver}

This section describes how the spectrum of a sheaf, or indeed a representation, over a quiver
changes under various combinatorial operations on the quiver.
Removing an edge or a vertex gives bounds on the change in spectrum, as Proposition~\ref{prop:non-section-preserving} describes,
however these may reduce the dimension of the space of sections.
Propositions~\ref{prop:edge-swap},~\ref{prop:enum-eigenvalues}, and~\ref{prop:schur-complement}
describe operations on the quiver,
dependent on the sheaf, that preserve the space of sections. These are described for sheaves over a quiver but all these operations also preserve quiver representations. These intermediate results are used to simplify quivers (while bounding the change in eigenvalues) in our Proofs of Theorems \ref{thm:computing-lim-e} and \ref{thm:computing-dual-lim-e}.

\begin{definition}
  \label{dfn:edge-pairing}
  For a quiver \(Q\),
  define an {\em edge pairing}
  as a collection \(P\) of pairs of distinct neighbouring edges
  \((e^1_1, e^1_2), \dots (e^m_1, e^m_2)\)
  of \(Q\)
  such no two edges are repeated, i.e. \(e^i_j \neq e^k_l\) unless \(i = k\) and \(j = l\).
  To simplify notation assume that each pair \((e^i_1, e^i_2)\), the edges \(e^i_1\) and \(e^i_2\) are directed away from their common vertex \(u^i\).
  Label the other vertices for \(e^i_1\) and \(e^i_2\),
  possibly the same vertex,
  as \(v^i\) and \(w^i\) respectively.
    \[
    \begin{tikzcd}
	& {\bullet^{u^i}} \\
	{\bullet^{v^i}} && {\bullet^{w^i}}
	\arrow["e^i_1"', no head, from=1-2, to=2-1]
	\arrow["e^i_2", no head, from=1-2, to=2-3]
      \end{tikzcd}
    \]
    We then define \(Q^P\)
    as the {\em P-homotopic quiver} to
    \(Q\)
    where \(Q^P\) has the same nodes and edges as \(Q\)
    except with each edge \(e^i_2\) replaced by an edge \(e^i_3\)
    between \(v^i\) and \(w^i\).
    \[
    \begin{tikzcd}
	& {\bullet^{u^i}} \\
	{\bullet^{v^i}} && {\bullet^{w^i}}
	\arrow["e^i_1"', no head, from=1-2, to=2-1]
	\arrow["e^i_3", no head, from=2-1, to=2-3]
      \end{tikzcd}
    \]
\end{definition}
    
    For \(\cA_\bullet\) is a sheaf over \(Q\)
    and a pairing \(P\) of \(Q\),
    say that \(\cA_\bullet\) is compatible with \(P\)
    if
    for each pair \((e^i_1, e^i_2) \in P\)
      we have that \(\cA_{e^i_1} = \cA_{e^i_2}\)
      and \(\cA_{u^i \to e^i_1} = \cA_{u^i \to e^i_2}\).
    Such a compatible sheaf induces a sheaf on \(Q^P\) as follows:
    define the sheaf \(\cA_\bullet^P\) on \(Q^P\)
    as the same as \(\cA_\bullet\) on the common edges of 
    \(Q\) and \(Q^P\)
    and otherwise
    by
    \(\cA^P_{e^i_3} = \cA_{e^i_2}\),
    \(\cA^P_{e^i_3 \leftarrow e^i_3} = \cA_{e^i_1 \leftarrow e^i_1}\)
    and
    \(\cA^P_{w^i \to e^i_3} = \cA_{e^i_2 \leftarrow w^i}\)
    for each \(i\).
    The next proposition shows the \(i\)-th eigenvalue of \(\cA^P_\bullet\)
    is bounded by a scalar multiple of the \(i\)-th eigenvalue 
    of \(\cA_\bullet\).
    First we will need a lemma.
    
\begin{lemma}
  \label{lem:golden-ratio}
  Suppose \(X, Y, Z \in \mbb{R}^n\).
  Then
  \[
    \|X - Z\|^2 + \|Z - Y\|^2 \leq \varphi^2 (\|X - Y\|^2 + \|Y - Z\|^2)
  \]
  where \(\varphi\) is the golden ratio \(\varphi = \frac{1 + \sqrt{5}}{2}\).
  This is also true for any metric space \((M, d)\) and can be written as
  \[
    d(X, Z)^2 \leq \varphi^2 d(X, Y)^2 + \varphi d(Y, Z)^2.
  \]
\end{lemma}
\begin{proof}
  Let us abbreviate \[x = \|Y - Z\| \qquad y = \|X - Z\| \qquad z = \|X - Y\|.\]
  If any of \(x, y\) or \(z\) are zero, then the proposition is immediately true,
  so assume that \(x, y\) and \(z\) are strictly positive.
  By the triangle inequality, we have \[\frac{x^2 + y^2}{x^2 + z^2} \leq \frac{x^2 + (x+z)^2}{x^2 + z^2}.\]
  We then want to know when the expression on the right side is maximised. As \(x, z > 0\), there exists \(u > 0\) for which \(x = u z\).
  Substituting this into the right side gives \[\frac{2u^2 + 2u + 1}{u^2 + 1},\]
  which is maximised when \(u = \varphi\). Thus, the maximum value of the ratio is \(\varphi^2\).
  If \(X, Y, Z\) are colinear with \(Y\) between \(X\) and \(Z\) and \(\|Y-Z\| = \varphi \|X-Y\|\), then the maximum is attained, which shows that this inequality is sharp.
\end{proof}

\begin{proposition}
  \label{prop:edge-swap}
  For \(Q\) a quiver and \(P\) an edge pairing of \(Q\),
  suppose that \(\bA_\bullet\) is a sheaf on \(Q\) compatible with \(P\).
    Then the eigenvalues of \(\cA_\bullet\) and \(\cA^P_\bullet\)
    are related by
    \[
      \varphi^{-2} \lambda_j(\cA_\bullet) \leq \lambda_j(\cA_\bullet^P) \leq \varphi^2 \lambda_j(\cA_\bullet)
    \] and vice-versa
    for \(j = 1, \dots, n\)
    where \(n\) is the total dimension of \(\bA_\bullet\).
\end{proposition}
\begin{proof}
  Using Lemma~\ref{lem:golden-ratio} we have
  \begin{align*}
    \mca{E}_{\cA}(x) = 
    & \sum_{(e^i_1, e^i_2) \in P}
      {\left[\left\|\cA_{e^i_1 \leftarrow v^i}x_{v^i} - \cA_{u^i \to e^i_1}x_{u^i}\right\|^2_{\cA_{e^i_1}}
      + \left\|\cA_{u^i \to e^i_2}x_{u^i} - \cA_{e^i_2 \leftarrow w^i}x_{w^i}\right\|^2_{\cA_{e^i_2}}\right]}
    \\
    & \hspace{5em} + \sum_{\textrm{unpaired } e \in Q_1}\left\|\cA_{s(e) \to e}x_{s(e)} - \cA_{e \leftarrow t(e)}x_{t(e)}\right\|^2_{\cA_{s(e)}} \\
    &\leq \varphi^2 \sum_{{(e^i_1, e^i_2) \in P}}
       {\left[\left\|\cA_{u^i \to e^i_1}x_{u^i} - \cA_{e^i_1 \leftarrow v^i}x_{v^i}\right\|^2_{\cA_{e^i_1}}
      + \left\|\cA^P_{e^i_3 \leftarrow v^i}x_{v^i} - \cA^P_{w^i \to e^i_3}x_{w^i}\right\|^2_{\cA_{e^i_3}}\right]} \\
    & \hspace{5em} + \sum_{\textrm{unpaired } e \in Q_1}\left\|\cA_{s(e) \to e}x_{s(e)} - \cA_{e \leftarrow t(e)}x_{t(e)}\right\|^2_{\cA_{s(e)}} \\
    &\leq \varphi^2 \mca{E}_{\cA^P}(x).
  \end{align*}
  Then by Theorem~\ref{thm:courant-fischer}
  \[
    \lambda_i(\cA_\bullet) = \min_{\dim{V} = i} \left[\max_{x \in V \setminus 0}\frac{\mca{E}_{\cA}(x)}{\langle x, x \rangle} \right]
     \leq \min_{\dim{V} = i} \left[\max_{x \in V \setminus 0}\frac{\varphi^2\mca{E}_{\cA^P}(x)}{\langle x, x \rangle} \right]
     = \varphi^2\lambda_i(\cA^P_\bullet).
   \]
   proving the proposition.
 \end{proof}
 
 We note that Proposition~\ref{prop:edge-swap} directly applies to graphs. Although we make no use of the following corollary, we hope that it will be of independent interest to spectral graph theorists.
 \begin{corollary}
   Suppose \(G = (V, E)\) is a multigraph with \(n\) vertices.
  Suppose we have an edge pairing \(P = (e^1_1, e^1_2), \dots (e^m_1, e^m_2)\) on $G$ as in Definition \ref{dfn:edge-pairing}. Let \(\widehat{G}\) with the same nodes and edges as \(G\)
    except with each edge \(e^i_2\) replaced by an edge \(e^i_3\)
    between \(v^i\) and \(w^i\). Then the eigenvalues of \(L_{\widehat{G}}\) and \(L_G\)
    are related by
    \[
      {\varphi^{-2}} \lambda_j(L_G) \leq \lambda_j(L_{\widehat{G}}) \leq \varphi^2 \lambda_j(L_G)
    \] and vice-versa
    for \(j = 1, \dots, n\).
 \end{corollary}

\begin{proposition}
  \label{prop:enum-eigenvalues}
  Suppose \(\cA_\bullet\)  is a sheaf over a quiver \(Q\) with total dimension \(n\).
  Given a subquiver \(Q^\prime\), let \(\cA^\prime_\bullet\) denote the restriction of \(\cA_\bullet\) to $Q'$.  
  \begin{enumerate}[label={\theproposition(\alph*)}]
  \item   \label{prop:remove-identities}
    If \(e \in Q_1\) is an edge such that \(s(e) = t(e)\) and \(\cA_{s(e)} = \cA_{t(e)} = \id\) 
    then \(\lambda_i(\cA_\bullet^\prime) = \lambda_i(\cA_\bullet)\) for \(i = 1, \dots, n\)
    where \(\cA^\prime_\bullet\) is \(\cA_\bullet\) restricted to the subquiver \(Q^\prime\) of \(Q\)
    with the edge \(e\) removed.
    \item
      \label{prop:remove-double-edges}
      Suppose we have an edge pairing \(P = (e^1_1, e^1_2), \dots (e^m_1, e^m_2)\)
      of \(Q\) such that for each \(i\), the edges \(e^i_1, e^i_2\) have the same source and target vertices.
     Assume that for each $i$, there exists a linear map \(M_{i}: \cA_{e^i_2} \to \cA_{e^i_1}\) satisfying both
      \[\cA_{s(e^i_1) \to e^i_1} = M_{i}\cA_{s(e^i_2) \to e^i_2} \qquad \text{and} \qquad
      \cA_{e^i_1 \leftarrow t(e^i_1)} = M_{i}\cA_{e^i_2 \leftarrow t(e^i_2)}\]
      Let \(Q^\prime\) be the subquiver of \(Q\) with all the edges \(e^1_2, \dots, e^n_2\) removed,
      and let \(\cA^\prime_\bullet\) be the restriction of \(\cA_\bullet\) to \(Q^\prime\).
      Then, we have
      \[
        (1+\Gamma_P)^{-1} \lambda_i(\cA_\bullet) \leq \lambda_i(\cA_\bullet^\prime) \leq \lambda_i(\cA_\bullet)
      \]
      for all \(i = 1, \dots, n\), where $\Gamma_P := \max_i \{\|M_{i}\|^2\}$.
  \end{enumerate}
\end{proposition}
\begin{proof}
  For \ref{prop:remove-identities}, observe that \(\mca{E}_{\cA}(x) = \mca{E}_{\cA^\prime}(x)\)
  and thus the result follows from Theorem~\ref{thm:courant-fischer}.
  For \ref{prop:remove-double-edges}, we have
  \begin{align*}
    \mca{E}_{\cA}(x) &
    \leq \sum_{(e^i_1, e^i_2) \in P}
      {(1 + \|M_{i}\|^2)\left\|\cA_{s(e^i_1) \to e^i_1}x_{s(e^i_1)} - \cA_{e^i_1 \leftarrow t(e^i_1)}x_{t(e^i_1)}\right\|^2}  \\
    & \hspace{5em} + \sum_{\textrm{unpaired } e \in Q_1}\left\|\cA_{s(e) \to e}x_{s(e)} - \cA_{e \leftarrow t(e)}x_{t(e)}\right\|^2_{\cA_{s(e)}} \\
    & \leq (1 + \Gamma_P) \mca{E}_{\cA^\prime}(x).
  \end{align*}
  and the result follows from Theorem~\ref{thm:courant-fischer}.
\end{proof}

In the following, at the risk of causing ambiguity with self-loops, we will ignore the direction of edges in and use the relation \(v \sim e\) if either \(s(e) = v\) or \(t(e) = v\).

\begin{definition}
\label{dfn:reduced-quiver}
    Let \(Q\) be a quiver
    and let \(V\) be a subset of vertices such that there are no edges in \(Q\) between vertices of \(V\). The {\bf V-reduced quiver} \(Q^V\) is defined as follows. For each vertex \(v \in V\), 
    \begin{enumerate} 
    \item first remove \(v\) and any associated edges from \(Q\);
\item then, for each pair of edges \((p, q)\) in $Q$ such that \(u \sim p \sim v \sim q \sim w\)
for vertices \(u, w \in Q_0 \setminus V\),
add a new edge \(pq\) between \(u\) and \(w\).
\end{enumerate}
(We note in the second step above that if \(u = w\), then $pq$ forms a self-loop.) The resulting quiver is the desired \(Q^V\).
\end{definition}

\begin{figure}[H]
  \caption{An example of the reduction of a quiver \(Q\) with \(V = \set{v}\).}
\[\begin{tikzcd}[ampersand replacement=\&]
	u \&\& v \&\& w \&\& u \&\&\&\& w \\
	\&\& \boxed{Q} \&\& \&\& \&\& \boxed{Q^{\set{v}}} \&\&
	\arrow[bend left=40, from=1-1, to=1-3, "p"]
	\arrow[bend right=40, from=1-1, to=1-3, "q"]
	\arrow[from=1-3, to=1-5, "r"]
	\arrow[bend left=20, from=1-7, to=1-11, "pr"]
	\arrow[bend right=20, from=1-7, to=1-11, "qr"]
        \arrow[loop, distance=2em, in=55, out=125, from=1-7, to=1-7, "pp"]
        \arrow[loop, distance=2em, in=145, out=215, from=1-7, to=1-7, "pq"]
        \arrow[loop, distance=2em, in=235, out=305, from=1-7, to=1-7, "qq"]
        \arrow[loop, distance=2em, in=325, out=35, from=1-11, to=1-11, "rr"]
\end{tikzcd}\]
\end{figure}

We now
describe an analogue of the Kron reduction of a graph~\cite{dorfler2010synchronization}
for sheaves over a quiver.
As discussed in~\cite{hansenSpectralTheoryCellular2019},
Kron reduction
does not work in general 
for sheaves over a (loopless) graph.
However, as the next proposition shows,
if one restricts the type of vertex
removed, dependent on the 
sheaf, and 
moreover allows for loops
in the underlying quiver, then
Kron reduction is possible.

\begin{proposition}
  \label{prop:schur-complement}
  Given a sheaf \(\cA_\bullet\) over a quiver \(Q\),   assume that \(V \subset Q_0\) is a subset of vertices such that
  there are no edges in \(Q\) between vertices of \(V\). Further suppose that for each \(v \in V\)
  there exists some real number \(\alpha_v > 0\)
  such that for each edge \(e \sim v\),
  we have 
  \[
  \text{either } \quad \cA_{v \to e}^* \cA_{v \to e} = \alpha_v \id_v \quad \text{ or } \quad \cA_{e \leftarrow v}^* \cA_{e \leftarrow v} = \alpha_v \id_v\]
  as appropriate. Then there exists a sheaf \(\cA^V_\bullet\) over \(Q^V\) such that \(\Gamma(\cA^V_\bullet) \cong \Gamma(\cA_\bullet)\)
  and
  \[
    \lambda_i(\cA_\bullet) \leq \lambda_i(\cA^V_\bullet) \leq \lambda_{i + r}(\cA_\bullet)
  \]
  for \(i = 1, \dots, n - r\)
  where \(n\) is the total dimension of \(\cA_\bullet\) and
  \(r = \sum_{v \in V}{\dim \cA_v}\).
\end{proposition}
\begin{proof}
    Consider the Laplacian \(L_{\cA}\).
  It can be written in a block form as
  \[
    L_{\cA} = 
  \begin{blockarray}{ccc}
   Q_0 \setminus V & V \\
  \begin{block}{(cc)c}
    X & Y^* & Q_0 \setminus V \\
    Y & D & V \\
  \end{block}
  \end{blockarray}.
\]
As there are no edges between vertices of \(V\),
\(D\) is a diagonal matrix where the \((v, v)\)
block is \(d_v \alpha_v \id\)
where \(d_v\) is the degree of \(v\).
Thus \(D\) is invertible, and we can form the Schur complement
\[
 L_{\cA} / D := X - Y^* D^{-1} Y.
\]

A vector \(\begin{bmatrix} x & z \end{bmatrix}^T\)
satisfies
\[
      \begin{bmatrix}
        X & Y^\dagger \\
        Y & D \\
      \end{bmatrix}
      \begin{bmatrix}
        x \\
        z \\
      \end{bmatrix}
      =
      \begin{bmatrix}
        0 \\
        0 \\
      \end{bmatrix}
 \]
    
    if and only if \(X x + Y^\dagger z = 0\) and \(Y x + D z = 0\).
    As \(D\) is invertible this can be rearranged to
    \(X x - Y D^{-1} Y^\dagger x = 0\)
    hence \(\ker{(L_{A} / D)} \cong \ker{L_{A}}\).
    For the interlacing, we can directly apply Theorem~\ref{thm:schur-interlacing}.

We will now define the sheaf \(\cA^V_\bullet\).
As in Definition~\ref{dfn:reduced-quiver},
we will ignore orientation of the edges
and write \(u \sim e\) if either \(s(e) = v\) or \(t(e) = v\)
and \(\cA_{v \sim e}\) for the corresponding assigned linear map.
If \(p \neq q\), define \(\cA^V_{u \sim pq} = (d_v\alpha_v)^{-1/2}\cA_{v \sim p}^*\cA_{u \sim p}\)
and likewise for \(w\).
Otherwise, if \(p = q\), define \(\cA^V_{u \to pp} = Z_p\cA_{u \sim p}\)
and \(\cA^V_{pp \leftarrow u} = \mbf{0}\)
for a linear map \(Z_p\) which we will now define.
First observe that for every non-zero vector \(x\), we have
\[
  \frac{1}{\alpha_v}x^*\cA_{v \sim p}\cA_{v \sim p}^*x \leq \frac{1}{\alpha_v}\|\cA_{v \sim p}^*\|^2 \|x\|^2
  = \frac{1}{\alpha_v}\|\cA_{v \sim p}^*\cA_{v \sim p}\| \|x\|^2
  = \|x\|^2
\]
thus there exists some \(Z_p\) such that
\[
  \id - \frac{1}{\alpha_v}\cA_{v \sim p}\cA_{v \sim p}^* = Z_p^*Z_p.
\]
since the left-hand side is positive semi-definite.

Now we claim that \(L_{\cA^V} = L_{\cA} / D\).
To do so, consider vertices
\(u \in Q_0 \setminus V\)
and \(v \in V\).
For each edge \(u \sim p \sim v\),
the contribution to the diagonal block indexed by \((u, u)\)
of \(L_{\cA}/D\) is
\[
    \cA_{u \sim p}^* \cA_{u \sim p} - \frac{1}{d_v\alpha_v} \sum_{u \sim q \sim v}\cA^*_{u \sim p}\cA_{v \sim p}\cA^*_{v \sim q}\cA_{u \sim q}.
\]
We need to show this is equal to the contribution
from the new edges of \(\cA^Q_\bullet\).
For each vertex \(u \in Q_0 \setminus V\),
the contribution to the diagonal block is
\begin{align*}
  & \sum_{v \in V}\sum_{u \sim p \sim v}\left[
    \cA_{u \sim p}^* \cA_{u \sim p} - \frac{1}{d_v\alpha_v} \sum_{u \sim q \sim v}\cA^*_{u \sim p}\cA_{v \sim p}\cA^*_{v \sim q}\cA_{u \sim q}
    \right] \\
  &= \sum_{v \in V} \bigg[
    \sum_{u \sim p \sim v}\left[\cA_{u \sim p}^* \cA_{u \sim p} - \frac{d^u_v}{d_v\alpha_v}\cA^*_{u \sim p}\cA_{v \sim p}\cA^*_{v \sim p}\cA_{u \sim p}
    \right] \\
  &\phantom{\sum_{u \sim p \sim v}}+
    \sum_{\substack{u \sim p \sim v \\ u \sim q \sim v \\ p \neq q}} \frac{1}{d_v\alpha_v}
    (\cA^*_{v \sim p}\cA_{u \sim p} - \cA^*_{v \sim q}\cA_{u \sim q})^*(\cA^*_{v \sim p}\cA_{u \sim p} - \cA^*_{v \sim q}\cA_{u \sim q})
    \bigg]
\end{align*}
where \(d_v^u\) is the number of edges between \(u\) and \(v\).
The term on the last line 
corresponds to the new self-loops on \(u\) given by \(v\).
By the defining property of \(Z_p\), the term
\[
   \sum_{u \sim p \sim v}\left[\cA_{u \sim p}^* \cA_{u \sim p} - \frac{d^u_v}{d_v\alpha_v}\cA^*_{u \sim p}\cA_{v \sim p}\cA^*_{v \sim p}\cA_{u \sim p}\right]
\]
equals
\[
  \sum_{u \sim p \sim v} \left[\frac{d_v - d^u_v}{d_v\alpha_v}\cA^*_{u \sim p}\cA_{v \sim p}\cA^*_{v \sim p}\cA_{u \sim p}
    + \cA^*_{u \sim p}Z_p^*Z_p\cA_{u \sim p}\right]
\]
and thus corresponds with the contribution of the self-loop \(pp\) and the new edges to different vertices,
as required.
\end{proof}

\section{Reduction of the Feature Selector Problem}
Using the quiver operations in \S\ref{sec:eigenvalues},
we can reduce the number of vertices in the quiver
whilst preserving the space of sections; this
allows us to reduce the difficulty of computing approximate sections. 

\label{sec:reduction}
\begin{theorem}
  \label{thm:computing-lim-e}
  Let \(\bS\) be a feature selector on \(X\), \(\cU\) a cover of \(X\) such that \(Q^\cU\) is (weakly) connected, and \(\cF \subset X^*\).
  Fix a vertex \(\sigma_0 \in Q^\cF_0\).
  Then there exists a quiver \(Q^\prime\) 
  and a representation \(\bA^\prime_\bullet\) of \(Q'\) satisfying the following properties: 
  \begin{enumerate}
  \item The quiver $Q'$ has only one vertex and \(|Q^\cF_0|\) edges.
   \item The Laplacian of $\bA'_\bullet$ is:
  \[
    \sum_{\tau \in Q^\cU_0}{\left[\id_{\sigma_0} - \pi^\cF_{\sigma_0} \iota^\cF_\tau \pi^\cF_\tau \iota^\cF_{\sigma_0}\right]}
  \]  
      \item There exist constants \(C_1, C_2 > 0\) that depend on \(\bS, \cU\) and \(\cF\)
  such that
  \[
    C_1 \lambda_i(\Comb^{\bS, \cF}_\bullet)
    \leq \lambda_i(\bA^\prime_\bullet)
    \leq C_2 \lambda_{i + k}(\Comb^{\bS, \cF}_\bullet)
  \]
  for \(i = 1, \dots, n - k\).
  \end{enumerate}
  Here \(n\) is the total dimension of \(\Comb^{\bS, \cF}_\bullet\)
  and \(k := n - \dim{\bA^{\bS, \cF}_{\sigma_0}}\).
   
\end{theorem}
\begin{proof}
  We view \(\Comb^{\bS, \cF}_\bullet\) as a sheaf \(\cA_\bullet\) over a quiver \(Q\), defined as follows:
  \begin{align*}
& Q_0 = \set{v_\sigma : \sigma \in Q^\cU_0} \cup \set{v_{\sigma \tau} : (\sigma \to \tau) \in Q^\cU_1}, \\
& Q_1 = \set{e_{\sigma \tau}, r^{DL}_{\sigma\tau}, r^{LD}_{\sigma\tau} : (\sigma \to \tau) \in Q^\cU_1}, \\
&s(e_{\sigma \tau}) = v_\sigma, \quad t(e_{\sigma \tau}) = v_\tau, \\
&s(r^{DL}_{\sigma \tau}) = s(r^{LD}_{\sigma \tau}) = v_{\sigma}, \quad t(r^{DL}_{\sigma \tau}) = t(r^{LD}_{\sigma \tau})= v_{\sigma \tau}.
  \end{align*}
This is precisely the merged quiver on which $\Comb^{\bS,\cF}_\bullet$ is defined, except that we have dispensed with identity self-loops in light of Proposition~\ref{prop:remove-identities}. The desired sheaf \(\cA_\bullet\) assigns the following data to vertices and edges of \(Q\) for each $\sigma \in Q^\cU_0$:
  \begin{align*}
\cA_{v_\sigma \to e_{\sigma \tau}} &= \pi^\cF_\tau \circ \iota^\cF_\sigma, &
\cA_{e_{\sigma \tau} \leftarrow v_\tau} &= \id,  &
\cA_{v_\sigma \to r^{DL}_{\sigma\tau}} &= \iota^\cF_\tau \circ \pi^\cF_\tau \circ \iota^\cF_\sigma, \\
\cA_{r^{DL}_{\sigma\tau} \leftarrow v_\sigma} &= \id, &
\cA_{v_\sigma \to r^{LD}_{\sigma\tau}} &= \iota^\cF_\sigma, &
\cA_{r^{DL}_{\sigma\tau} \leftarrow v_\sigma} &= \id.
\end{align*}
  
  Use Proposition~\ref{prop:remove-double-edges}
  to get a new copy \(e^2_{\sigma\tau}\) of the edge \(e_{\sigma\tau}\) with the maps multiplied by \(\iota^\cF_\tau\).
  Use this edge with Proposition~\ref{prop:edge-swap} to move the edge \(r^{DL}_{\sigma\tau}\) to an edge between \(v_\tau\) and \(v_{\sigma\tau}\),
  then using the edge \(r^{LD}_{\sigma\tau}\) with Proposition~\ref{prop:edge-swap} move this edge between \(v_\tau\) and \(v_{\sigma\tau}\) to
  an edge between \(v_\sigma\) and \(v_\tau\) which has maps \(\iota_\sigma\) and \(\iota_\tau\) respectively.
  Call this edge \(p_{\sigma\tau}\).
  Using Proposition~\ref{prop:remove-double-edges} can remove both \(e_{\sigma\tau}\) and \(e^2_{\sigma\tau}\) since \(\pi^\cF_\tau\)
  is the left inverse of \(\iota^\cF_\tau\).
  Finally, we can remove the edge \(r^{LD}_{\sigma\tau}\) with Proposition~\ref{prop:schur-complement}
  as \((\iota^\cF_\sigma)^* \iota^\cF_\sigma = \id\).
  Now choose some vertex \(\sigma_0\).
  As the nerve is connected, and each edge from each vertex has the same restriction maps to that vertex
  we can keep applying Proposition~\ref{prop:edge-swap} until all edges have \(\sigma_0\) as a boundary.
  There might be double edges, but these can be removed with Proposition~\ref{prop:remove-double-edges}.
  Then we apply Proposition~\ref{prop:schur-complement} to reduce to just the vertex \(\sigma_0\).
  The Laplacian is now
  \[
    \sum_{\tau \in Q^\cU_0}{\left[\id_{\sigma_0} - \pi^\cF_{\sigma_0} \iota^\cF_\tau \pi^\cF_\tau \iota^\cF_{\sigma_0}\right]}
  \]
  and the eigenvalue relationship follows by tracking all the uses of Propositions~\ref{prop:remove-identities}, \ref{prop:remove-double-edges}, \ref{prop:edge-swap},  \ref{prop:schur-complement}.
\end{proof}

\begin{theorem}
  \label{thm:computing-dual-lim-e}
  Let \(\bS\) be a feature selector on a set \(X\) equipped with a cover \(\cU\). For each subset 
  \(\cF \subset X^*\), here exists a sheaf \(\cA^\prime\) over a quiver \(Q^\prime\) with Laplacian
  \[
    (L_{\cA^\prime})_{\sigma, \tau} =
    \begin{cases}
      \sum_{\tau \subset \sigma}{(2\id_{\sigma} - \pi^\cF_\sigma\iota^\cF_\tau \pi^\cF_\tau\iota^\cF_\sigma)}
      + \sum_{\sigma \subset \gamma}{\pi^\cF_\sigma\iota^\cF_{\gamma}\pi^\cF_{\gamma}\iota^\cF_\sigma} & \text{if } \sigma = \tau, \\
      - \pi^\cF_{\sigma}\iota^\cF_{\tau} & \text{if } \tau \subset \sigma \text{ or } \sigma \subset \tau, \\
      \mbf{0} & \text{otherwise.}
    \end{cases}
  \]
  such that there is an isomorphism \(\Gamma(\cA^\prime_\bullet) \cong \Gamma(\MixedComb^{\bS, \cF}_\bullet)\); and moreover, we have
  \[
    \lambda_i(\MixedComb^{\bS, \cF}_\bullet)
    \leq \lambda_i(\cA^\prime_\bullet)
    \leq \lambda_{i + k}(\Comb^{\bS, \cF}_\bullet)
  \]
  for all \(i \in \set{1, \dots, n - k}\). Here $\MixedComb^{\bS,\cF}_\bullet$ is the mixed representation of Definition \ref{def:dual-combinedrep}, while \(n\) is its total dimension 
  and \(k := n - \sum_{\sigma \in Q_0}{\dim{\bA^{\bS, \cF}_\sigma}}\).
\end{theorem}
\begin{proof}
    Consider the representation \((\bA^{\bS, \cF}_\bullet)^* \sqcup_R \widehat{\bA}^{\bS, \cF}_\bullet\)
  as defined in Definition~\ref{def:dual-combinedrep},
  and view it as a sheaf \(\cA_\bullet\) over a quiver \(Q\).
  We can immediately remove identity self-loops using Proposition~\ref{prop:remove-identities}.
  The quiver \(Q\) is as follows:
  \begin{align*}
& Q_0 = \set{v_\sigma : \sigma \in Q^\cU_0} \cup \set{v_{\sigma \tau} : (\sigma \to \tau) \in Q^\cU_1},\\
& Q_1 = \set{e_{\sigma \tau}, r^{DL}_{\sigma\tau}, r^{LD}_{\sigma\tau} : (\sigma \to \tau) \in Q^\cU_1},\\
&s(e_{\sigma \tau}) = v_\tau, \quad t(e_{\sigma \tau}) = v_\sigma,\\
&s(r^{DL}_{\sigma \tau}) = s(r^{LD}_{\sigma \tau}) = v_{\sigma}, \quad t(r^{DL}_{\sigma \tau}) = t(r^{LD}_{\sigma \tau})= v_{\sigma \tau}.
  \end{align*}
  
  The sheaf \(\cA_\bullet\) consists of the following for each \(\sigma \in Q^\cU\):
  \begin{align*}
\cA_{v_\tau \to e_{\sigma \tau}} &= \pi^\cF_\sigma \circ \iota^\cF_\tau, &
\cA_{e_{\sigma \tau} \leftarrow v_\sigma} &= \id, \\
\cA_{v_\sigma \to r^{DL}_{\sigma\tau}} &= \iota^\cF_\tau \circ \pi^\cF_\tau \circ \iota^\cF_\sigma, &
\cA_{r^{DL}_{\sigma\tau} \leftarrow v_\sigma} &= \id, \\
\cA_{v_\sigma \to r^{LD}_{\sigma\tau}} &= \iota^\cF_\sigma, 
&\cA_{r^{DL}_{\sigma\tau} \leftarrow v_\sigma} &= \id.
  \end{align*}
  We can use Proposition~\ref{prop:schur-complement} to remove each \(v_{\sigma\tau}\).
  In doing so, we add three new self-loops to \(v_\sigma\); two identities and a self loop with maps \(\iota^\cF_\sigma\) and \( \iota^\cF_\tau \circ \pi^\cF_\tau \circ \iota^\cF_\sigma\).
  The two identity self loops can be removed with Proposition~\ref{prop:remove-identities}.
\end{proof}

\subsection{Remarks on computation}
In Theorems \ref{thm:computing-lim-e} and \ref{thm:computing-dual-lim-e},
we have assumed that \(L\) is Hermitian with respect to some orthonormal basis.
Often, however, the bases we would prefer to use are not orthogonal with respect to our chosen inner product.
We will see such an example in \S\ref{sec:sc-atac},
where computing inner products is straightforward whilst finding an orthonormal basis is computationally infeasible.
It is then advantageous then to compute eigenvectors of a Laplacian without such a change of basis.
Suppose, for each \(\sigma \in \cU\), we have some basis \(\cB_\sigma\) of \(\bS^\cF_{|\sigma|}\)
and a Hermitian positive definite matrix \(M^{\bS, \cF}_\sigma\) representing the inner product in this basis.
As \(M^{\bS, \cF}_\sigma\) is positive definite there exists an invertible \(W_\sigma\) such that \(M^{\bS, \cF}_\sigma = W^*_\sigma W_\sigma\). 
Then, in this basis, we now have for \(\sigma, \tau \in \cU\)
\begin{align*}
  &\pi^\cF_\sigma\iota^\cF_\tau = W^{-*}_\sigma M^{\bS, \cF}_{\sigma \tau} W^{-1}_\tau, \\
  &\pi^\cF_\sigma \iota^\cF_\tau \pi^\cF_\tau \iota^\cF_\sigma = W^{-*}_\sigma M^{\bS, \cF}_{\sigma \tau} (M^{\bS, \cF}_{\tau})^{-1} M^{\bS, \cF}_{\tau \sigma} W^{-1}_\sigma
\end{align*}
where \(M^{\bS, \cF}_{\sigma \tau}\) is the matrix of inner products between \(\cB_\sigma\) and \(\cB_\tau\).
Now computing eigenvectors of a Laplacian \(L\) in this basis
can be realised as the generalised eigenvector problem
\[
 L x = \lambda M^{\bS, \cF}_{\cU} x
\]
where \(M^{\bS, \cF}_{\cU}\) is the block diagonal matrix with blocks \(M^{\bS, \cF}_\sigma\) for \(\sigma \in \cU\).

As the matrices \(M^{\bS, \cF}_{\tau}\) are Hermitian and positive-definite,
the products
\[M^{\bS, \cF}_{\sigma \tau} (M^{\bS, \cF}_{\tau})^{-1} M^{\bS, \cF}_{\tau \sigma}\]
can be quickly computed using the Cholesky decomposition of \(M^{\bS, \cF}_{\tau}\).
Furthermore, as \(L\) is Hermitian and usually sparse, there exist computationally efficient methods for calculating
its smallest eigenvalues and their eigenvectors,
e.g. variations on the Lanczos method~\cite{lanczos1950iteration}, combined with the Cholesky decomposition (see e.g.~\cite{hornMatrixAnalysis2017})
of \(L + \alpha \mbf{I}\) for some small \(\alpha > 0\).

\section{Single-Cell Chromatin Accessibility}
\label{sec:sc-atac}

{\em Chromatin} is a macromolecule consisting of DNA and certain proteins, primarily histones \cite{cooper2022cell}. Histones facilitate the dense compaction of DNA, which allows it to fit inside limited space within the cell nucleus. The fundamental unit of chromatin is a {\em nucleosome}, which consists of approximately 147 {\em base-pairs} (namely, the usual A, C, G and T nucleotides) wound around histone molecules. The density of nucleosomes varies considerably across the genome; low density regions are more accessible to transcription factors and other binding proteins. Thus, understanding the relative accessibility of various parts of the genome is important in the study of gene regulation and its variation across different cell types.

Among the most promising experimental frameworks for measuring chromatin accessibility is {\bf ATAC-seq} \cite{buenrostro2013transposition},
which has been extended to single cell sequencing {\bf scATAC-seq} (see e.g.~\cite{satpathyMassivelyParallelSinglecell2019}).
The output of scATAC-seq may be viewed from a mathematical perspective as follows. Consider a set $X$ of cells \(\{x_1, \dots, x_N\}\), with $N$ being on the order of thousands. Each cell $x_i$ has a genome consisting of $M$ base pairs, where $M$ is approximately three billion. The output of ATAC-seq on $X$ is an $N \times M$ integer matrix. For each $1 \leq i \leq N$ and $1 \leq j \leq M$, the entry \(\alpha_{i,j}\) is a number in \(\{0, 1, 2\}\) representing the accessibility (for cell $x_i$ at the $j$-th genomic position). Such data is high-dimensional, noisy and sparse, and hence warrants feature selection. This is accomplished via a class of techniques called {\bf peak-calling} algorithms \cite{yan2020reads}. The goal of these methods is to identify specific locations along the genome, called {\em summits}, which are significantly accessible across the cells in $X$ --- see for instance the MACS2 algorithm \cite{zhang2008model}.

\begin{remark} We are often interested in quantifying accessibility when one restricts to various sub-populations $Y \subset X$ of cell types. However, the peaks of various $Y$ may be lost when calling peaks on $X$, particularly when the cardinalities of the $X$ under consideration are unbalanced.
The standard remedy is to first cluster the data and then call peaks on each cluster \cite{satpathyMassivelyParallelSinglecell2019}. Subsequently, peaks from all the clusters are combined, with various rules beings applied for grouping and extending peaks which overlap. We show below how the quiver Laplacian provides a different mechanism for achieving similar results.
\end{remark}

\subsection{Pre-processing the data}\label{ssec:preproc}

We use the single-cell ATAC-seq data of \cite{satpathyMassivelyParallelSinglecell2019}, which consist of 28,274 tumour-infiltrating T cells. The pre-processing steps are the same as those of ibid:
\begin{enumerate}
\item We align fragments to the \textsc{GRCh37/hg19} assembly.
\item We filter sample profiles to keep those with at least 1,000 fragments and a transcription start site (TSS) enrichment score of 8.
\item We create a count matrix by tiling the genome with 2,500 base-pair windows and counting overlaps of endpoints of fragments with the windows.
\item Next, we binarise this matrix and retain only the top 20,000 most accessible sites; call the resulting $28,724 \times 20,000$ boolean matrix $B$. 
\item We then create a frequency-inverse document frequency matrix, apply 
truncated singular value decomposition to obtain the 2nd to 25th singular vectors.
This results in a reduced matrix $E$ of size $28724 \times 24$.
\end{enumerate}

\subsection{Generating the quiver}\label{ssec:makequiv}

We build an undirected graph $G$ from the reduced matrix $E$ as follows. Its vertices are the cells $C_i$, which correspond to rows of $E$. Treating each such row as a vector in $\R^{24}$,  let $\delta_{i,j} \geq 0$ be the Euclidean distance between rows $i$ and $j$ for each $i < j$. We construct the 15 nearest neighbour graph generated by the $\delta_{i,j}$ distances. Our analysis now deviates from \cite{satpathyMassivelyParallelSinglecell2019}. Using the Leiden algorithm~\cite{traag2019louvain} the vertex set $V$ of $G$ is partitioned into finitely many communities, $V = \coprod_i V_i$. As in ibid, communities of size $<$ 200 are excluded. We then generate a cover $\cU$ of $X$ from the surviving communities as follows: each cover element $U_i \in \cU$ has the form $V_i \cup W_i$, where $W_i \subset (V - V_i)$ consists of those vertices which share an edge with some vertex of $V_i$. The resulting quiver $Q^\cU$ has
 87 vertices, and is depicted in Figure~\ref{fig:results-overview}A.
Similarly, Figure~\ref{fig:results-overview}B displays the contents of each vertex relative to the overall distribution,
using the cell types identified in~\cite{satpathyMassivelyParallelSinglecell2019}.

\subsection{Building the representation}\label{ssec:buildrep}

Having constructed $Q^\cU$, we apply our feature selector $\bS$ -- which in this case is the MACS2 peak-calling algorithm \cite{zhang2008model} -- to the fragment profiles of cells corresponding to each individual vertex. In the notation of Section \ref{sec:feature-selection}, the input $\cF$ is given by the columns of the binarised matrix $B$; explicitly, the map $f_j:X \to \R$ associated to the $j$-th column of $E$ sends each $x_i$ to the entry $B_{ij} \in \set{0,1}$. Calling peaks on this set of features produces a vector space $\bS_{|\sigma|} \subset \cV(\cF)$ for every vertex $\sigma$ of $Q^\cU$. We used the same parameters as in~\cite{satpathyMassivelyParallelSinglecell2019}, and obtained several hundred thousand summits for each vertex as a result. For the purposes of this demonstration, we retained only the most significant 1,000 summits per vertex. Thus the total dimension of the selected vector spaces is 87,000. Since peak-calling on different sub-populations may change the true summit location on the genome slightly, we follow \cite{yuSingleCellATACseqAnalysis2022} and use a Gaussian kernel between summits (with bandwidth equal to 1,000 base-pairs) as the inner product on $\cV(\cF)$. This allows us to build the mixed representation $\MixedComb^{\bS,\cF}_\bullet$ (see Definition \ref{def:dual-combinedrep}).

\subsection{Results}\label{ssec:peakres}

We computed the first 3,000 eigenvectors of the Laplacian of $\MixedComb^{\bS,\cF}_\bullet$. Since the total space admits a distinguished basis whose vectors correspond to summits (i.e., certain genomic positions $\set{p_k}$), we may associate to each eigenvector $v = \sum_{k}\alpha_kp_k$ the subset of supported genomic positions  
$P_v := \set{k \mid \alpha_k \neq 0}$\footnote{Or more precisely, $|\alpha_k| > 10^{-5}$, the tolerance parameter for \textsc{ARPACK-NG}~\cite{lehoucq2023arpack} in our experiment.}.
Figure~\ref{fig:results-overview}D shows that for an overwhelming majority of eigenvectors $v$, 
the diameter of $P_v$ for most $v$ approximately equals 147 base-pairs, which is the average length of DNA around a nucleosome. A small fraction of the eigenspaces had dimension $> 1$; this resulted in a mixing of genomic positions, as shown by the blue outliers in Figure~\ref{fig:results-overview}D. Nevertheless, in almost all cases, we observed that {\bf the genomic supports of eigenvectors are tightly aligned across the vertices} of $Q^\cU$. 

We then performed a more detailed analysis of the sparsity pattern of the 3,000 eigenvectors. The outcome of this analysis has been summarised in 
Figure~\ref{fig:results-overview}E, which plots the 3,000 eigenvectors against all 87 vertices of $Q^\cU$ and highlights those vertices along which a given eigenvector has a non-zero component. We observe that the eigenvectors fall broadly into two classes. The first class, which is considerably larger, consists of those  
{\bf eigenvectors whose genomic support contains summits which are shared by most, if not all vertices}: see, for instance, the components of eigenvector 11 in Figure~\ref{fig:eigenvector-comparison}. The second class contains {\bf eigenvectors whose genomic support contains summits which are shared by a small fraction of the vertices}. A prime example of this type is eigenvector 2702, which has nonzero components along only 3 of the 87 vertices, and has also been depicted in Figure~\ref{fig:eigenvector-comparison}. 

Finally, we note that (at least for these two examples) {\bf the size of the support of the eigenvector relates to biological function}. Eigenvector 11, for instance,  is concentrated in a regulatory region a few hundred base pairs downstream of the gene \textit{FAM72D}. This gene is involved in
the cell cycle \cite{fischer2016integration}, and is therefore relevant to most subpopulations of cells.
On the other hand, eigenvector 2702 lies just downstream of the gene \textit{STAT3}, which is a transcription factor
known to be an important regulator
for differentiation and proliferation in
Th1, Th17, and Treg cells \cite{rebe2019stat3}. We also note that this eigenvector 
was only supported on the vertices 
8, (8, 11), and 11 (Figure~\ref{fig:eigenvector-comparison}). These vertices
are particularly enriched in Th1, Th17 and Treg cells (see Figure~\ref{fig:results-overview}B).
Indeed, Figure~\ref{fig:results-overview}E
reveals that
 many other eigenvectors 
are also uniquely supported on the vertices
8, (8, 11), 11 and (1, 11).
Since vertex 11 is relatively
disconnected from the other vertices,
our analysis highlights it as an unusual sub-population of $X$.

\begin{figure}[ht]
  \caption{
  Overview of Laplacian from~\ref{thm:computing-dual-lim-e} applied
  to tumour-infiltrating T cells.
  (A) Quiver constructed from cover,
  with vertex size corresponding to the cardinality of the covering set.
  (B) Log likelihood of each cell type label
  within each vertex.
  (C) First 3,000 eigenvalues.
  (D) Range of genomic positions on each chromosome
  corresponding to non-zero components of the first 3,000 eigenvectors.
  Blue points correspond to positions across multiple chromosomes.
  (E) Binary matrix corresponding to the existence
  of a non-zero component in each vertex
  for the first 3,000 eigenvectors.}
  \label{fig:results-overview}
  \includegraphics[height=0.8\textheight]{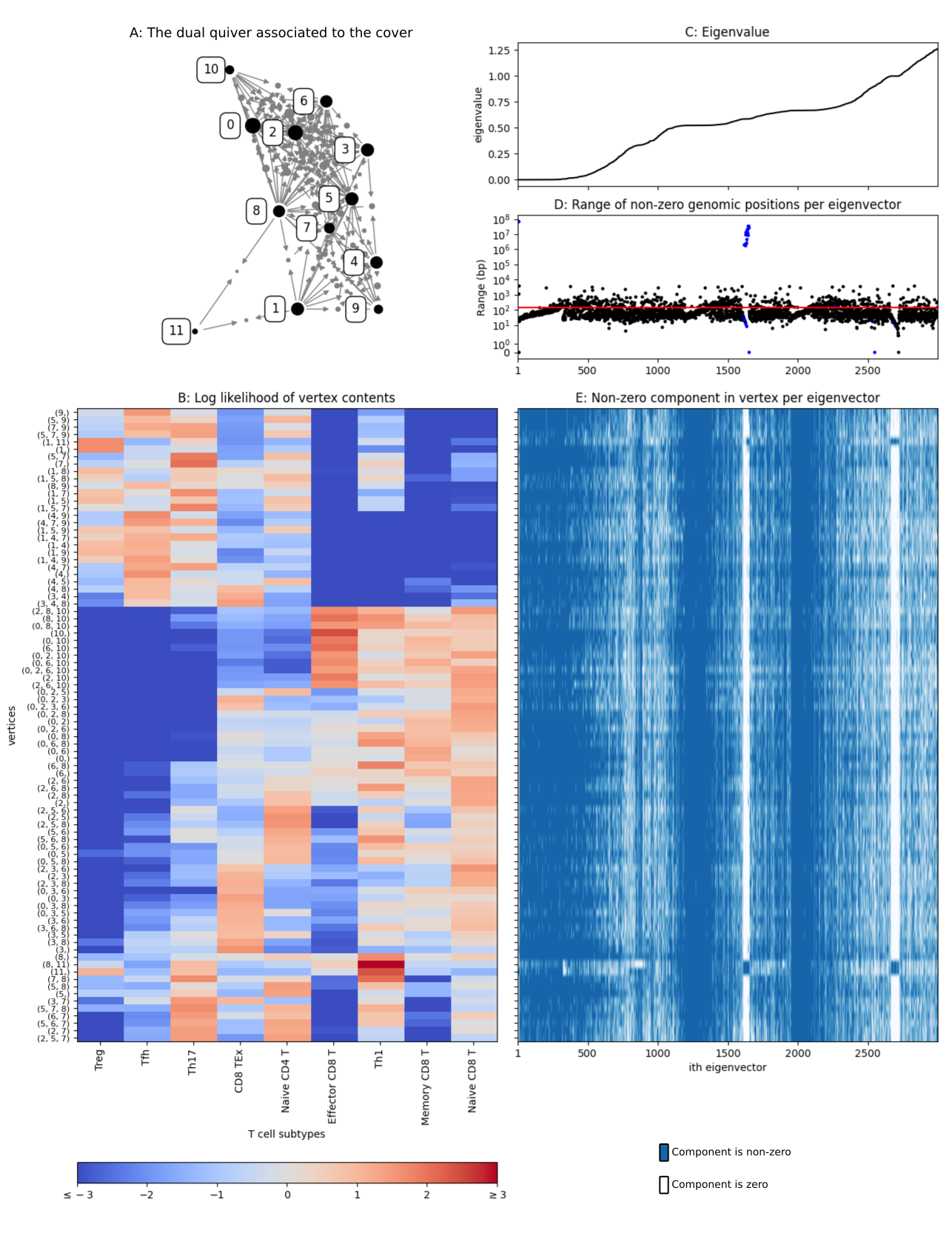}
  \centering
\end{figure}

\begin{figure}[ht]
  \caption{Examples of eigenvectors.
  Eigenvector \#11 has a non-zero component at each vertex,
  whilst eigenvector \#2702 has only non-zero components in vertices (8,), (11,) and (8, 11).}
  \label{fig:eigenvector-comparison}
  \includegraphics[height=.8\textheight]{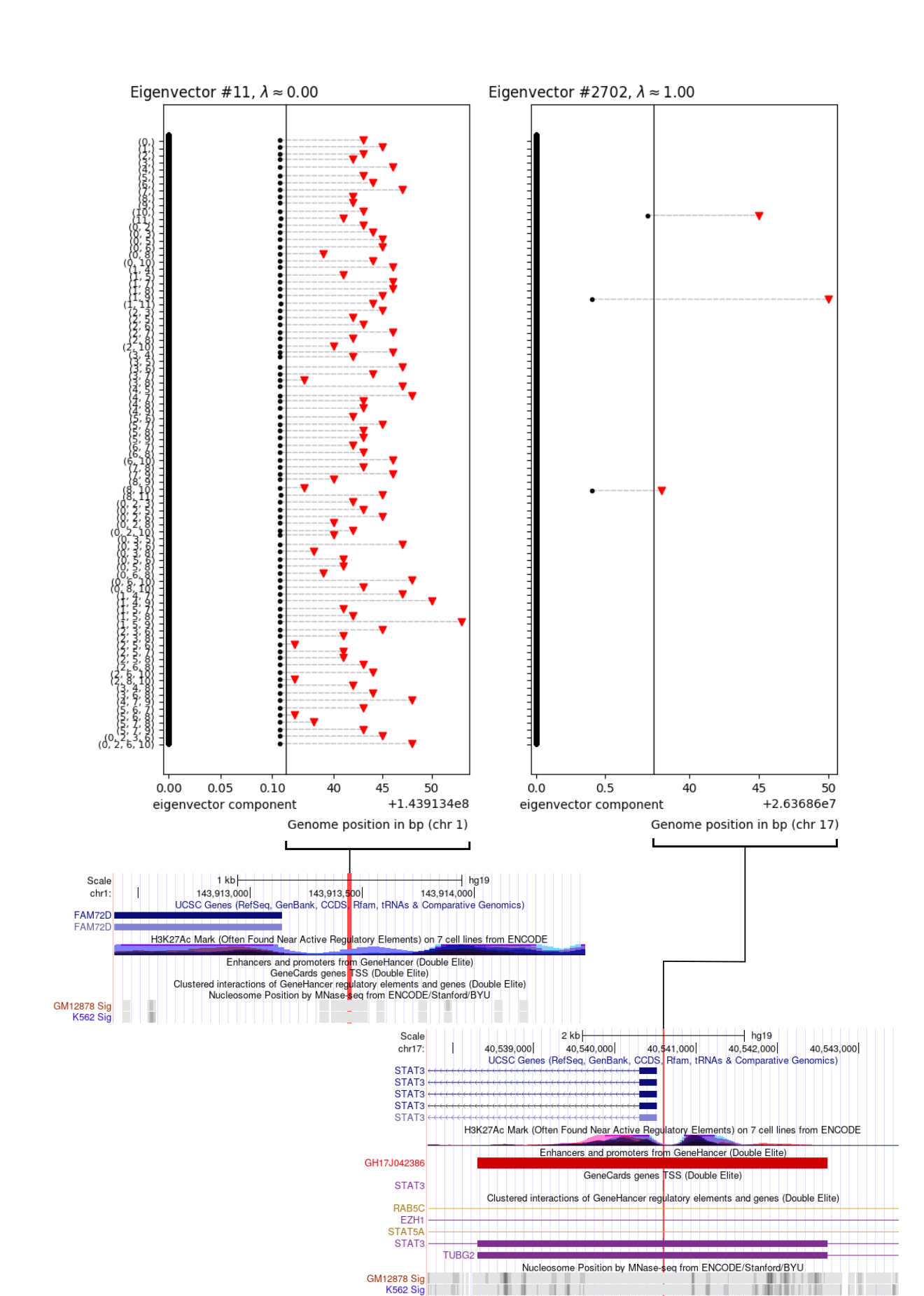}
  \centering
\end{figure}

\section*{Code and data availability}
Code produced for the analysis will be 
available at 
\newline
\centerline{
\href{https://github.com/osumray/harmonic_feature_selection}{\url{https://github.com/osumray/harmonic_feature_selection}}.}
The data from~\cite{satpathyMassivelyParallelSinglecell2019} is
available from the Gene Expression Omnibus with accession number \href{https://www.ncbi.nlm.nih.gov/geo/query/acc.cgi?acc=GSE129785}{GSE129785}.

\appendix
\section{Removing Edges and Vertices}
\label{sec:additional-proofs}

The following extends known results from graphs, e.g. see~\cite{van1995hamilton} and~\cite{wuDeletingVerticesInterlacing2010} respectively.
\begin{proposition}
  Let \(\cA_\bullet\) be a sheaf over a quiver \(Q\),
  and let \(n\) be its total dimension.
  \label{prop:non-section-preserving}
  \begin{enumerate}[label={\theproposition(\alph*)}]
  \item \label{prop:removing-edges}
    Let \(R \subset Q_1\) be a subset of edges of \(Q\).
    Let \(Q^\prime\) be the quiver with edges \(Q_1 \setminus R\)
    and \(\cA^\prime\) the restriction of \(\cA\) to \(Q^\prime\).
    Then
    \[
      \lambda_{k + i}(\cA_\bullet) \leq \lambda_{k + i}(\cA_\bullet^\prime)
      \leq \lambda_{k + r + i}(\cA_\bullet)
    \]
    for \(i = 1, \dots, n - k - r\)
    where \(k = \dim \Gamma(\cA_\bullet)\) and
    \(r = \sum_{e \in R}{\dim \cA_{e}}\).
    \item \label{prop:removing-vertices}
  Suppose \(V \subset Q_0\) is a subset of vertices of \(Q\).
  Let \(Q^\prime\) be the quiver with vertices \(Q_0 \setminus V\)
  and all edges \(e \in Q_1\) such that either \(s(e) \in V\) or \(t(e) \in V\),
  and \(\cA_\bullet^\prime\) the restriction of \(\cA_\bullet\) to \(Q^\prime\).
  Then
  \[
    \lambda_i(\cA_\bullet) - w_1 \leq \lambda_i(\cA^\prime_\bullet) \leq \lambda_{k + i}(\cA_\bullet) - w_2
  \]
  for \(i = 1, \dots, n - k\)
  where \(k = \sum_{u \in Q_0 \setminus V}{\dim \cA_u}\)
  and
  \begin{align*}
    w_1 &= \max_{v \in V}\Big\{
      \sum_{\substack{s(e) = v \\ t(e) \notin V}}\sigma_{\max}(\cA_{v \to e})^2
  + \sum_{\substack{t(e) = v \\ s(e) \notin V}}{\sigma_{\max}(\cA_{e \leftarrow v})^2}\Big\}, \\
 w_2 &= \min_{v \in V}\Big\{ \sum_{\substack{s(e) = v \\ t(e) \notin V}}\sigma_{\min}(\cA_{v \to e})^2
  + \sum_{\substack{t(e) = v \\ s(e) \notin V}}{\sigma_{\min}(\cA_{e \leftarrow v})^2}\Big\}.
  \end{align*}
  \end{enumerate}
\end{proposition}

\begin{proof}[Proof of Proposition~\ref{prop:removing-edges}]
  Define \(L_{\cA, \op} = B_{\cA} B_{\cA}^*\).
  Let \(k = \dim \ker L_{\cA}\) and \(c = \dim \ker L_{\cA, \op}\).
  Then for \(i = 1, \dots, n - k\),
  we have that
  \[
    \lambda_{k + i}(L_{\cA}) = \lambda_{c + i}(L_{\cA, \op}).
  \]
  Now, removing the set of edges \(R \subset Q_1\)
  corresponds to the $r \times r$ principal submatrix of \(L_{\cA, \op}\), where \(r := \sum_{e \in R}{\dim \cA_{e}}\). By Theorem~\ref{thm:cauchy-interlacing} we have
  \[
    \lambda_i(L_{\cA, \op}) \leq \lambda_i(L_{\cA^\prime, \op}) \leq \lambda_{i + r}(L_{\cA, \op})
  \]
  where \(i \in \set{1, \dots, m - r}\) thus
  \[
    \lambda_{k + i}(\cA_\bullet) \leq \lambda_{k + i}(\cA^\prime_\bullet) \leq \lambda_{k + r + i}(\cA_\bullet)
  \]
  for \(i \in \set{1, \dots, n - k - r}\).
\end{proof}

\begin{proof}[Proof of Proposition~\ref{prop:removing-vertices}]
  The boundary matrix \(B_{\cA}\) has block form
  \[
  \begin{bmatrix}
    B_{11} & \mbf{0} \\
    B_{21} & B_{22}
  \end{bmatrix}
  \]
  where \(B_{11} = B_{\cA^\prime}\).
  Here, \(B_{\cA}\) has row blocks indexed by edges between vertices in \(V\) and column blocks indexed by vertices in \(V\).
  Then  Laplacian \(L_{\cA} = B_{\cA}^* B_{\cA}\) has block form
  \[
    \begin{bmatrix}
      B_{11}^* B_{11} + B_{21}^*B_{21} & B_{21}^* B_{22} \\
      B_{22}^* B_{21} & B_{22}^* B_{22}
    \end{bmatrix}.
  \]
  Define \(L_{Q^0 \setminus V}=  B_{11}^* B_{11} + B_{21}^*B_{21}\) which is a principal submatrix of \(L_{\cA}\).
  Note that \(B_{21}^*B_{21}\) is a diagonal matrix:
  The \(v, v\) block is
  \[
    \sum_{\substack{s(e) = v \\ t(e) \neq V}}{\cA_{v \to e}^*\cA_{v \to e}}
    + \sum_{\substack{t(e) = v \\ s(e) \neq V}}{\cA_{e \leftarrow v}^*\cA_{e \leftarrow v}}.
  \]
  
We now have that
\begin{align*}
  \lambda_i(\cA^\prime_\bullet) &= \lambda_i(B_{11}^* B_{11}) = \lambda_i(L_{Q_0 \setminus V} - B_{21}^*B_{21}) 
                                              \leq \lambda_i(L_{Q_0 \setminus V}) + \lambda_{n - k}(-B_{21}^* B_{21}) \\
                                              &= \lambda_i(L_{Q_0 \setminus V}) - \lambda_{1}(B_{21}^* B_{21}) \\
                                              &= \lambda_i(L_{Q_0 \setminus V}) - \min_{v \in V}{\lambda_{1}\Big(
      \sum_{\substack{s(e) = v \\ t(e) \neq V}}{\cA_{v \to e}^*\cA_{v \to e}}
  + \sum_{\substack{t(e) = v \\ s(e) \neq V}}{\cA_{e \leftarrow v}^*\cA_{e \leftarrow v}}\Big)}\\
                                              &\leq \lambda_i(L_{Q_0 \setminus V}) - \min_{v \in V}\Big\{{
      \sum_{\substack{s(e) = v \\ t(e) \neq V}}\lambda_{1}({\cA_{v \to e}^*\cA_{v \to e}})
  + \sum_{\substack{t(e) = v \\ s(e) \neq V}}{\cA_{e \leftarrow v}^*\cA_{e \leftarrow v}}}\Big\}\\
                                                &= \lambda_i(L_{Q_0 \setminus V}) - \min_{v \in V}\Big\{{
      \sum_{\substack{ s(e) = v \\ t(e) \neq V}}\sigma_{\min}(\cA_{v \to e})^2
  + \sum_{\substack{t(e) = v \\ s(e) \neq V}}{\sigma_{\min}(\cA_{e \leftarrow v})^2}\Big\}}\\
                                              &\leq \lambda_{i + k}(L_{\cA}) - \min_{v \in V}\Big\{{
      \sum_{\substack{s(e) = v \\ t(e) \neq V}}\sigma_{\min}(\cA_{v \to e})^2
  + \sum_{\substack{t(e) = v \\ s(e) \neq V}}{\sigma_{\min}(\cA_{e \leftarrow v})^2}\Big\}} .
\end{align*}
The first two inequalities are from Theorem~\ref{thm:weyl-inequalities} and the final inequality is due to Theorem~\ref{thm:cauchy-interlacing}.
The right inequality is proven similarly.
\end{proof}
\section{Classical Results on Eigenvalues of Hermitian Matrices}

In the following, for an \(n \times n\) Hermitian matrix \(A\),
let \(\lambda_1(A), \dots, \lambda_n(A)\) be the eigenvalues
ordered in increasing fashion.

\begin{theorem}[Rayleigh, see~\cite{hornMatrixAnalysis2017} Theorem 4.2.2]
  \label{thm:rayleigh}
  Let \(A\) be a Hermitian matrix.
  Then for integers \(1 \leq i_1, \dots, i_k \leq n\)
  let \(x_{i_1}, \dots, x_{i_k}\) be orthonormal,
  where \(x_{i_j}\) is an eigenvector corresponding to \(\lambda_{i_j}(A)\).
  Define \(S = \Span{\set{x_{i_1}, \dots, x_{i_k}}}\).
  Then for any \(x \in S\)
  \[
    \lambda_{i_1}(A) \leq \frac{x^* A x}{\|x\|^2} \leq \lambda_{i_k}(A).
  \]
\end{theorem}

\begin{theorem}[Courant-Fischer, see~\cite{hornMatrixAnalysis2017} Theorem 4.2.6]
  \label{thm:courant-fischer}
  Let A be Hermitian matrix.
  Then
  \[
    \lambda_k(A) = \min_{\dim{V} = k}\left[\max_{x \in V \setminus 0} \frac{\langle x, Ax \rangle}{\langle x, x \rangle} \right].
  \]
\end{theorem}

\begin{theorem}[Cauchy's interlacing theorem, see~\cite{hornMatrixAnalysis2017} Theorem 4.3.17]
  \label{thm:cauchy-interlacing}
  Let \(A\) be a Hermitian matrix of size \(n \times n\)
  and \(B\) an \(m \times m\) principal submatrix of \(A\).
  Then
  \[
    \lambda_i(A) \leq \lambda_i(B) \leq \lambda_{i + n - m}(A)
  \]
  for \(i = 1, \dots, m\).
\end{theorem}

\begin{theorem}[Weyl's inequalities, see~\cite{hornMatrixAnalysis2017} Theorem 4.3.1]
  \label{thm:weyl-inequalities}
  For \(A, B\) Hermitian matrices of size \(n \times n\)
  the following inequalities hold:
  \[
    \lambda_i(A + B) \leq \lambda_{i + j}(A) + \lambda_{n - j}(B)
  \]s
  for \(j = 0, \dots, n - 1\) and
  \[
    \lambda_i(A + B) \geq \lambda_{i - j + 1}(A) + \lambda_{j}(B)
  \]
  for \(j = 1, \dots, i\).
\end{theorem}

\begin{theorem}[See \cite{smithInterlacingPropertiesSchur1992} Theorem 5]
  \label{thm:schur-interlacing}
  For \(H\) a Hermitian matrix of dimension \(n\) given in block form as
  \[
    H = 
    \begin{bmatrix}
      A & B^* \\
      B & D \\
    \end{bmatrix}
  \]
  where \(D\) is non-singular
  let \(H / D = A - B D^{-1} B^*\) be the Schur complement of D.
  Then the eigenvalues of \(H\) and \(H / D\) are interlaced:
  \[
    \lambda_{i}(H) \leq \lambda_{i}(H / D) \leq \lambda_{i + r}(H)
  \]
  for \(i = 1, \dots, n - r\)
  where \(r\) is the rank of \(D\).
\end{theorem}

\bibliographystyle{abbrv}
\bibliography{main}

\end{document}